\documentclass[a4paper]{article}
\usepackage{amsmath,amsthm,amsfonts,thmtools}
\usepackage{amssymb}
\usepackage{amsmath}
\usepackage{amsthm}
\usepackage{amsfonts}
\usepackage{amssymb}
\usepackage{calligra}
\usepackage{calrsfs}
\usepackage{enumitem}
\usepackage{euscript}
\usepackage{array, booktabs, makecell}
\usepackage{booktabs}
\usepackage{appendix}
\usepackage[usenames,dvipsnames]{color}
\usepackage[table, dvipsnames]{xcolor}
\definecolor{Gray}{gray}{0.95}
\usepackage{url}

\usepackage{multirow}

\usepackage[font=footnotesize,labelfont=bf]{caption}
\usepackage[font=footnotesize,labelfont=bf]{subcaption}

\usepackage[usenames]{color}
\usepackage{listings} 	
\usepackage{graphicx}
\usepackage{subcaption}
\usepackage{mwe}

\DeclareMathAlphabet{\pazocal}{OMS}{zplm}{m}{n}

\numberwithin{equation}{section}
\newtheorem{theorem}{Theorem}[section]

\newtheorem{lemma}[theorem]{Lemma}

\newtheorem{algorithm}[theorem]{Algorithm}

\numberwithin{equation}{section}

\theoremstyle{definition}

\newtheoremstyle{myremarkstyle}{}{}{}{}{\bfseries}{.}{ }{}
\theoremstyle{myremarkstyle}
\declaretheorem[name=Remark,qed=$\blacksquare$,numberlike=theorem]{remark}

\makeatletter
\newcommand*{\intavg}{%
  \mint@l{-}{}%
}
\newcommand*{\mint@l}[2]{%
  \@ifnextchar\limits{%
    \mint@l{#1}%
  }{%
    \@ifnextchar\nolimits{%
      \mint@l{#1}%
    }{%
      \@ifnextchar\displaylimits{%
        \mint@l{#1}%
      }{%
        \mint@s{#2}{#1}%
      }%
    }%
  }%
}
\newcommand*{\mint@s}[2]{%
  \@ifnextchar_{%
    \mint@sub{#1}{#2}%
  }{%
    \@ifnextchar^{%
      \mint@sup{#1}{#2}%
    }{%
      \mint@{#1}{#2}{}{}%
    }%
  }%
}
\def\mint@sub#1#2_#3{%
  \@ifnextchar^{%
    \mint@sub@sup{#1}{#2}{#3}%
  }{%
    \mint@{#1}{#2}{#3}{}%
  }%
}
\def\mint@sup#1#2^#3{%
  \@ifnextchar_{%
    \mint@sub@sup{#1}{#2}{#3}%
  }{%
    \mint@{#1}{#2}{}{#3}%
  }%
}
\def\mint@sub@sup#1#2#3^#4{%
  \mint@{#1}{#2}{#3}{#4}%
}
\def\mint@sup@sub#1#2#3_#4{
  \mint@{#1}{#2}{#4}{#3}%
}
\newcommand*{\mint@}[4]{%
  \mathop{}%
  \mkern-\thinmuskip
  \mathchoice{%
    \mint@@{#1}{#2}{#3}{#4}%
        \displaystyle\textstyle\scriptstyle
  }{%
    \mint@@{#1}{#2}{#3}{#4}%
        \textstyle\scriptstyle\scriptstyle
  }{%
    \mint@@{#1}{#2}{#3}{#4}%
        \scriptstyle\scriptscriptstyle\scriptscriptstyle
  }{%
    \mint@@{#1}{#2}{#3}{#4}%
        \scriptscriptstyle\scriptscriptstyle\scriptscriptstyle
  }%
  \mkern-\thinmuskip
  \int#1%
  \ifx\\#3\\\else_{#3}\fi
  \ifx\\#4\\\else^{#4}\fi  
}
\newcommand*{\mint@@}[7]{%
  \begingroup
    \sbox0{$#5\int\m@th$}%
    \sbox2{$#5\int_{}\m@th$}%
    \dimen2=\wd0 %
    \let\mint@limits=#1\relax
    \ifx\mint@limits\relax
      \sbox4{$#5\int_{\kern1sp}^{\kern1sp}\m@th$}%
      \ifdim\wd4>\wd2 %
        \let\mint@limits=\nolimits
      \else
        \let\mint@limits=\limits
      \fi
    \fi
    \ifx\mint@limits\displaylimits
      \ifx#5\displaystyle
        \let\mint@limits=\limits
      \fi
    \fi
    \ifx\mint@limits\limits
      \sbox0{$#7#3\m@th$}%
      \sbox2{$#7#4\m@th$}%
      \ifdim\wd0>\dimen2 %
        \dimen2=\wd0 %
      \fi
      \ifdim\wd2>\dimen2 %
        \dimen2=\wd2 %
      \fi
    \fi
    \rlap{%
      $#5%
        \vcenter{%
          \hbox to\dimen2{%
            \hss
            $#6{#2}\m@th$%
            \hss
          }%
        }%
      $%
    }%
  \endgroup
}

\def\XXint#1#2#3{{\setbox0=\hbox{$#1{#2#3}{\int}$ }
		\vcenter{\hbox{$#2#3$ }}\kern-.6\wd0}}

\renewcommand{\geq}{\geqslant}
\renewcommand{\leq}{\leqslant}

\renewcommand{\epsilon}{\varepsilon}
\renewcommand{\phi}{\varphi}

\newcommand{\R}{\mathbb{R}}
\newcommand{\N}{\mathbb{N}}

\newcommand{\bu}{{\bf u}}		

\newcommand{\train}{\EuScript{S}}

\newcommand{\er}{\EuScript{E}}

\newcommand{\dom}{\mathbb{D}}

\newcommand{\res}{\EuScript{R}}

\newcommand{\domain}{D_T\times S \times \Lambda}

\begin{document}

\date{\today}

\title{Physics Informed Neural Networks for \\Simulating Radiative Transfer}

\author{Siddhartha Mishra, Roberto Molinaro  \thanks{Seminar for Applied Mathematics (SAM), D-Math \newline
  ETH Z\"urich, R\"amistrasse 101. }}

\date{\today}

\maketitle

\begin{abstract}
We propose a novel machine learning algorithm for simulating radiative transfer. Our algorithm is based on physics informed neural networks (PINNs), which are trained by minimizing the residual of the underlying radiative transfer equations. We present extensive experiments and theoretical error estimates to demonstrate that PINNs provide a very easy to implement, fast, robust and accurate method for simulating radiative transfer. We also present a PINN based algorithm for simulating inverse problems for radiative transfer efficiently.
\end{abstract}

\section{Introduction}
The study of radiative transfer is of vital importance in many fields of science and engineering including astrophysics, climate dynamics, meteorology, nuclear engineering and medical imaging \cite{Modbook}. The fundamental equation describing radiative transfer is a \emph{linear partial integro-differential equation}, termed as the \emph{radiative transfer equation}. Under the assumption of a static underlying medium, it has the following form \cite{Modbook}, 
\begin{equation}
\label{eq:genRTE}
\begin{aligned}
  \frac{1}{c}u_t +   \omega\cdot\nabla_x u  + ku + \sigma\Bigg(u - \frac{1}{s_d}\int\limits_{\Lambda}\int\limits_{S} \Phi(\omega, \omega^{\prime}, \nu, \nu^{\prime})u(t,x,\omega^{\prime}, \nu^{\prime}) d\omega^{\prime}d\nu^{\prime} \Bigg)= & f,
\end{aligned}
\end{equation}
with time variable $t \in [0,T]$, space variable $x \in D \subset \R^d$ (and $D_T = [0,T] \times D$), \emph{angle} $\omega \in S = {\mathbb S}^{d-1}$ i.e. the $d$-dimensional sphere and \emph{frequency} (or group energy) $\nu \in \Lambda \subset \R$. The constants in \eqref{eq:genRTE} are the speed of light $c$ and the surface area $s_d$ of the $d$-dimensional unit sphere. The unknown of interest in \eqref{eq:genRTE} is the so-called \emph{radiative intensity} $u: D_T \times S \times \Lambda \mapsto \R$, while $k = k(x,\nu): D \times \Lambda \mapsto \R_+$ is the \emph{absorption coefficient} and $\sigma =  \sigma(x,\nu): D \times \Lambda \mapsto \R_+$ is the \emph{scattering coefficient}. The integral term in \eqref{eq:genRTE} involves the so-called \emph{scattering kernel} $\Phi: S\times S \times \Lambda \times \Lambda \mapsto \R$, which is normalized as $\int_{S\times \Lambda} \Phi(\cdot 
,~\omega^{\prime},~\cdot,~\nu^{\prime}) d\omega^{\prime} d \nu^{\prime} = 1$, in order to account for the conservation of photons during scattering. The dynamics of radiative transfer are driven by a source (emission) term $f = f(x,\nu):  D \times \Lambda \mapsto \R$. 

Although the radiative transfer equation \eqref{eq:genRTE} is linear, explicit solution formulas are only available in very special cases \cite{Modbook}. Hence, numerical methods are essential for the simulation of the radiative intensity in \eqref{eq:genRTE}. However, the design of efficient numerical methods is considered to be very challenging \cite{Kanbook,KAU,Modbook}. This is on account of the \emph{high-dimensionality} of the radiative transfer equation \eqref{eq:genRTE}, where in the most general case of three space dimensions, the radiative intensity is a function of $7$ variables ($4$ for space-time, $2$ for angle and $1$ for frequency). Traditional grid-based numerical methods such as finite elements or finite differences, which involve $N^\ell$ degrees of freedom (for $\ell$ dimensions, with $N$ being the number of points in each dimension), require massive computational resources to be able to simulate radiative transfer accurately \cite{Kanbook,KAU}. Moreover in practice, one has to encounter media with very different optical properties characterized by different scales in the absorption and scattering coefficients and in the emission term in \eqref{eq:genRTE}, which further complicates the design of robust and efficient numerical methods. 

In spite of the aforementioned challenges, several types of numerical methods have been proposed in the literature for simulating radiative transfer, see \cite{Fran,GRELLA,Kanbook,Modbook} and references therein for a detailed overview. These include Monte Carlo ray-tracing type particle methods \cite{Stam1}, which do not suffer from the curse of dimensionality and are easy to parallelize but are characterized by slow convergence (with respect to number of particles) and are mostly limited to media with fairly uniform optical properties. \emph{Discrete Ordinate Methods} (DOM), are based on the discretization of the angular domain $S$ with a number of fixed directions and the resulting systems of spatio-temporal PDEs is solved by finite element or finite difference methods. Although easy to implement, these methods can be very expensive and also suffer from the so-called ray effects in optically thin media \cite{Lat}. Spherical harmonics, based on a series expansion in the angle, have been widely used in radiative transfer \cite{Modbook}. Although shown to exhibit spectral convergence for smooth solutions \cite{GrS}, these methods are well-known to still suffer from the curse of dimensionality, see \cite{GRELLA,ModYang}. A particular variant of the spherical harmonics, the so-called $P_1$ method, is an example of a class of flux limited diffusion methods \cite{Fran}, which are widely used for optically thick media. 

\emph{Moment} based methods lead to another class of numerical methods for simulating radiative transfer, see \cite{Fran} and references therein. For these methods, one derives a PDE for the so-called \emph{incident radiation} by integrating \eqref{eq:genRTE} over the angular domain $S$. The evolution of the incident radiation is determined by the \emph{heat (radiation) flux}, which is also the first angular moment. The evolution of the heat flux has to be determined from the second angular moment of $u$, which is termed as the \emph{pressure tensor}. These hierarchy of moments have to be closed by suitable closure relations (see \cite{Fran} and references therein) and the resulting PDEs are discretized by finite elements or finite differences. These moment based methods can lead to inaccurate approximation of the incident radiation, particularly when suitable closures are not available. Finally in recent years, several attempts have been made to design efficient finite element methods for the radiative transfer equation, such as those based on sparse grids \cite{Widmer} or sparse tensor product finite element spaces \cite{Kanbook,Widmer}. Although these methods can alleviate the curse of dimensionality in certain cases, they are rather complicated to implement and can still be computationally expensive, particularly when higher-order elements are used \cite{GRELLA}.

Summarizing the above discussion, it is fair to conclude that all the proposed methods have some deficiencies, in particular in their computational cost for simulating realistic problems. Thus, there is a pressing need to design a numerical method that is accurate, fast (in terms of computational time), easy to use and able to deal with the high dimensions and optical heterogeneity of the underlying radiative transfer equation \eqref{eq:genRTE}. We aim to propose such a numerical method in this article. 

Our proposed numerical method is based on deep neural networks \cite{DLbook}, i.e. functions formed by concatenated compositions of affine transformations and scalar non-linear activation functions. Deep neural networks have been extremely successful at diverse tasks in science and engineering \cite{DLnat} such as at image and text classification, computer vision, text and speech recognition, autonomous systems and robotics, game intelligence and even protein folding \cite{Dfold}. They are being increasingly used in the context of scientific computing, particularly for different aspects of numerical solutions of PDEs \cite{HEJ1,E1,LMPR1,LMR1} and references therein. 

As deep neural networks possess the so-called \emph{universal approximation property} or ability to accurately approximate any continuous (even measurable) function \cite{Bar1}, they can be used as ansatz (search) spaces for solutions of PDEs. This property lays the foundation for the so-called \emph{Physics informed neural networks} (PINNs) which collocate the PDE residual on \emph{training points} of the approximating deep neural network. First proposed in \cite{Lag2,Lag1}, PINNs been revived and developed in significantly greater detail recently in the pioneering contributions of Karniadakis and collaborators. PINNs have been successfully applied to simulate a variety of forward and inverse problems for PDEs, see \cite{KAR8, jag1, jag2, KAR9,KAR5,KAR6,KAR7,KAR1,KAR2,KAR4, shukla} and references therein.

In recent papers \cite{MM1} (for the forward problem) and \cite{MM2} (for the inverse problem), the authors analyzed PINNs and provide a rigorous explanation for the efficiency of PINNs, based on stability of the underlying PDEs. A surprising observation in \cite{MM1,MM2} was the ability of PINNs to overcome the curse of dimensionality, at least for some PDEs. This observation, together with the well-documented ability of PINNs to approximate PDEs is a starting point of this paper where we adapt PINNs to solve the radiative transfer equation \eqref{eq:genRTE}. The main contributions of the current paper are as follows,
\begin{itemize}
    \item We present a novel algorithm for approximating the radiative transfer equation \eqref{eq:genRTE} in a very general setting. Our algorithm is based on suitable physics informed neural networks (PINNs). 
    \item We analyze the proposed algorithm by rigorously proving an estimate on the so-called generalization error of the PINN. This estimate shows that as long as the PINN is trained well, it approximates the solution of \eqref{eq:genRTE} to high accuracy.
    \item We present a suite of numerical experiments to illustrate the accuracy and efficiency of the proposed algorithm. 
    \item A major advantage of PINNs is their ability to approximate inverse problems (with the same level of complexity as the forward problem). Hence, we will also modify PINNs to approximate an inverse problem for radiative transfer, namely determining the unknown absorption or scattering coefficients in \eqref{eq:genRTE} from measurements of moments of the radiative intensity.

\end{itemize}
Thus, we present a novel, fast, robust, accurate and easy to code and implement algorithm for simulating the general form of the  radiative transfer equations \eqref{eq:genRTE} and provide analysis and numerical experiments to demonstrate that this algorithm efficiently approximates both forward and inverse problems for radiative transfer. The rest of the paper is organized as follows, in section \ref{sec:2}, we describe the PINNs algorithm and provide an estimate on the underlying generalization error. Numerical experiments are presented in section \ref{sec:3}, PINNs for inverse problems are described in section \ref{sec:4} and the proposed method and results are discussed in section \ref{sec:5}.

\section{Physics informed neural networks for approximating \eqref{eq:genRTE}}
\label{sec:2}
In this section, we describe the PINNs algorithm for simulating radiative transfer. We start by elaborating on the underlying PDE \eqref{eq:genRTE}. 
\subsection{The model.}
\label{sec:21}
We model radiative transfer in a static medium by the evolution equation \eqref{eq:genRTE} for the radiative intensity $u$. This partial integro-differential equation is supplemented with the initial condition,
\begin{equation}
    \label{eq:ic}
     u(0, x,\omega,\nu) = u_0(x,\omega,\nu), \quad (x,\omega, \nu)\in D\times S\times \Lambda,
\end{equation}
for some initial datum $u_0: D \times S \times \Lambda \mapsto \R$.

Given that the radiative transfer equation \eqref{eq:genRTE} is a \emph{transport equation}, the boundary conditions are imposed on the so-called \emph{inflow boundary} given by,
\begin{equation}
    \Gamma_-= \{(t, x,\omega,\nu)\in [0,T]\times\partial D\times S \times \Lambda: \omega \cdot n(x) <0\}
\end{equation}
with $n(x)$ denoting the unit outward normal at any point $x \in \partial D$ (the boundary of the spatial domain $D$). We specific the following boundary condition, 
\begin{equation}
    \label{eq:bc}
 u(t,x,\omega,\nu) = u_b(t, x,\omega,\nu), \quad (t, x,\omega, \nu)\in\Gamma_-,
\end{equation}
for some boundary datum $u_b: \Gamma_- \mapsto \R$.

Given that the radiative intensity is a function of $2d+1$-variables, it is essential to find suitable low-dimensional functionals (observables) to visualize and interpret it. To this end, one often considers physically interesting angular-moments such as the \emph{incident radiation} (zeroth angular moment) and \emph{heat flux} (first angular moment) given by,
\begin{equation}
\label{eq:inc_rad}
    G(t, x, \nu) = \int\limits_S u(t, x,\omega, \nu)d\omega
\end{equation}
\begin{equation}
    F(t, x, \nu) = \int\limits_S u(t,x,\omega, \nu)\omega d\omega
\end{equation}
 
We note that for many applications of radiative transfer, it is common to consider the steady (time-independent) version of the radiative transfer equation \eqref{eq:genRTE}, which formally results from setting $c \rightarrow \infty$ and dropping the time-derivative term in the left hand side of \eqref{eq:genRTE}.

\subsection{Quadrature rules and Training points}
\label{sec:22}
Quadrature i.e numerical approximation of integrals, is essential for simulating the radiative transfer equation \eqref{eq:genRTE} with PINNs. It is needed for approximating the integral with the scattering kernel in \eqref{eq:genRTE}. Moreover, we follow \cite{MM1}, where quadrature points were used as training points for PINNs. 

Given any domain $\dom$ and an integrable function $g: \dom \mapsto \R$, we need to specify quadrature points $z_i \in \dom$ for $1 \leq i \leq N$, and quadrature weights $w_i$ in order to perform the following approximation,
\begin{equation}
    \label{eq:quad1}
    \int\limits_{\dom} g(z) dz \approx \sum\limits_{i=1}^N w_i g(z_i). 
\end{equation}

For our specific integrals, we consider \emph{Gauss-Legendre} quadrature rules \cite{SBbook} for approximating the scattering kernel integral in \eqref{eq:genRTE}. To this end, we choose points $z^S_i = (\omega^S_i,\nu^S_i)$, for $1 \leq i \leq N_S$, with $\omega^S_i \in S$ and $\nu^S_i \in \Lambda$ as the Gauss-Legendre quadrature points and the weights $w_i^S$ are the corresponding quadrature weights for a Gauss-Legendre quadrature rule of order $s \geq 1$.

We also need the following training points for the PINNs algorithm,
\subsubsection{Interior training points.}
We set $\train_{int} = \{z^{int}_j\}$, for $1 \leq j \leq N_{int}$, and $z_j^{int} = (t^{int}_j,x^{int}_j,\omega^{int}_j,\nu^{int}_j)$ with $t^{int}_j \in [0,T], x^{int}_j \in D, \omega^{int}_j \in S, \nu^{int}_j \in \Lambda$, for all $j$. These points are the quadrature points of a suitable quadrature rule with weights $w^{int}_j$.

If the underlying spatial domain $D \subset \R^d$ can be mapped to a $d$-dimensional rectangle, either entirely or in patches, then we can set the training points $z^{int}_j$ as a low-discrepancy Sobol sequence \cite{SOB} in $[0,1]^{2d+1}$, by rescaling the relevant domains. Sobol sequences arise in the context of Quasi-Monte Carlo integration \cite{CAF1} and the corresponding quadrature weights are $w^{int}_j \equiv \frac{1}{N_{int}}$, for all $j$. Note that the QMC quadrature rule does not suffer from the curse of dimensionality (see section \ref{sec:25} for details). In case the geometry of the domain is very complicated, one has simply choose random points, independent and identically distributed with the underlying uniform distribution, as training points. 
\subsubsection{Temporal boundary training points.}
We denote $\train_{tb} = \{z^{tb}_j\}$, for $1 \leq j \leq N_{tb}$, and $z_j^{tb} = (x^{tb}_j,\omega^{tb}_j,\nu^{tb}_j)$ with $x^{tb}_j \in D, \omega^{tb}_j \in S, \nu^{tb}_j \in \Lambda$, for all $j$. These points are the quadrature points of a suitable quadrature rule with weights $w^{tb}_j$. We can choose Sobol points for logically rectangular domains $D$ or random points to constitute this training set. 
\subsubsection{Spatial boundary training points.}
We denote $\train_{sb} = \{z^{sb}_j\}$, for $1 \leq j \leq N_{sb}$, and $z_j^{sb} = (t^{tb}_j, x^{tb}_j,\omega^{tb}_j,\nu^{tb}_j)$ with $t^{tb}_j \in [0,T], x^{tb}_j \in \partial D, \omega^{tb}_j \in S, \nu^{tb}_j \in \Lambda$, for all $j$. These points are the quadrature points of a suitable quadrature rule with weights $w^{sb}_j$. As before, we can choose Sobol points for logically rectangular domains $D$ or random points to constitute this training set. 
\subsection{Neural Networks}
\label{sec:23}
PINNs are neural networks i.e. given an input $y = (t,x,\omega,\nu) \in \dom=\domain$, a feedforward neural network (also termed as a multi-layer perceptron), shown in figure \ref{fig:1}, transforms it to an output, through a layers of units (neurons) which compose of either affine-linear maps between units (in successive layers) or scalar non-linear activation functions within layers \cite{DLbook}, resulting in the representation,
\begin{equation}
\label{eq:ann1}
\bu_{\theta}(y) = C_K \circ A \circ C_{K-1}\ldots \ldots \ldots \circ A \circ C_2 \circ A \circ C_1(y).
\end{equation} 
Here, $\circ$ refers to the composition of functions and $A$ is a scalar (non-linear) activation function (applied to vectors componentwise). A large variety of activation functions have been considered in the machine learning literature \cite{DLbook}. Popular choices for the activation function $A$ in \eqref{eq:ann1} include the sigmoid function, the hyperbolic tangent function and the \emph{ReLU} function.

The affine map in the $k$-layer is given by,
\begin{equation}
\label{eq:C}
C_k z_k = W_k z_k + b_k, \quad {\rm for} ~ W_k \in \R^{d_{k+1} \times d_k}, z_k \in \R^{d_k}, b_k \in \R^{d_{k+1}}.
\end{equation}
For consistency of notation, we set $d_1 = \bar{d} = 2d+1$, for $d$-space dimensions and $d_K = 1$. 

Thus in the terminology of machine learning (see also figure \ref{fig:1}), our neural network \eqref{eq:ann1} consists of an input layer, an output layer and $(K-1)$ hidden layers for some $1 < K \in \N$. The $k$-th hidden layer (with $d_k$ neurons) is given an input vector $z_k \in \R^{d_k}$ and transforms it first by an affine linear map $C_k$ \eqref{eq:C} and then by a nonlinear (component wise) activation $\sigma$. A straightforward addition shows that our network contains $\left(2d+2 + \sum\limits_{k=2}^{K-1} d_k\right)$ neurons. 
We also denote, 
\begin{equation}
\label{eq:theta}
\theta = \{W_k, b_k\},~ \theta_W = \{ W_k \}\quad \forall~ 1 \leq k \leq K,
\end{equation} 
to be the concatenated set of (tunable) weights for our network. It is straightforward to check that $\theta \in \Theta \subset \R^M$ with
\begin{equation}
\label{eq:ns}
M = \sum\limits_{k=1}^{K-1} (d_k +1) d_{k+1}.
\end{equation}
 \begin{figure}[htbp]
\centering
\includegraphics[width=8cm]{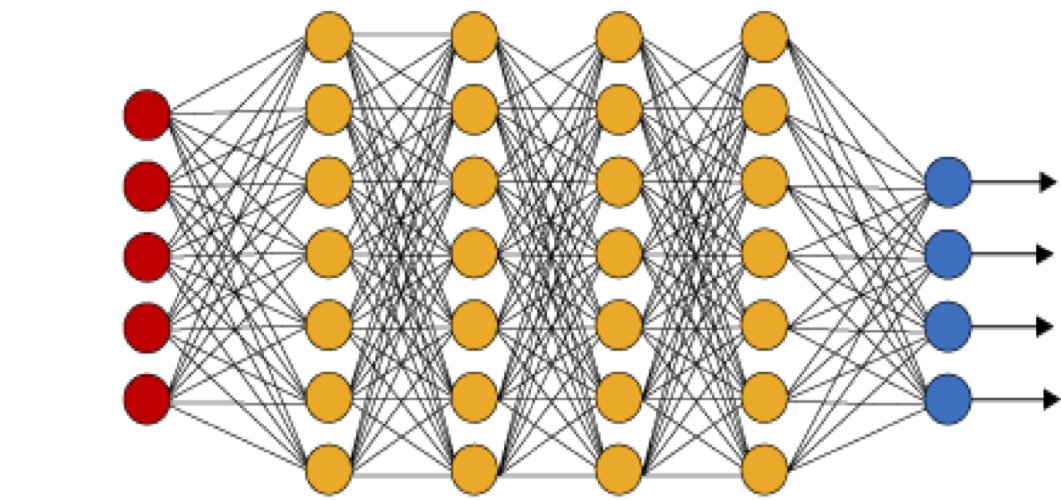}
\caption{An illustration of a (fully connected) deep neural network. The red neurons represent the inputs to the network and the blue neurons denote the output layer. They are
connected by hidden layers with yellow neurons. Each hidden unit (neuron) is connected by affine linear maps between units in different layers and then with nonlinear (scalar) activation functions within units.}
\label{fig:1}
\end{figure}
\subsection{Training PINNs: Loss functions and optimization}
\label{sec:24}
The neural network $u_{\theta}$ \eqref{eq:ann1} depends on the tuning parameter $\theta \in \Theta$ of weights and biases. Within the standard paradigm of \emph{deep learning} \cite{DLbook}, one \emph{trains} the network by finding tuning parameters $\theta$ such that the loss (error, mismatch, regret) between the neural network and the underlying target is minimized. Our target is the solution $u $ of the radiative transfer equation \eqref{eq:genRTE} and we wish to find the tuning parameters $\theta$ such that the resulting neural network $u_{\theta}$ approximates $u$. 

To do so, we follow \cite{Lag1,KAR1,MM1} and define the following \emph{PDE residual} $\res_{int,\theta} = \res_{int,\theta}(t,x,\omega,\nu)$, for all $(t,x,\omega,\nu) \in \domain$,
\begin{equation}
    \label{eq:res1}
    \res_{int,\theta}:= \frac{1}{c}\partial_t u_{\theta} +   \omega\cdot\nabla_x u_{\theta}  + ku_{\theta} + \sigma\Bigg(u_{\theta} - \frac{1}{s_d}
    \sum\limits_{i=1}^{N_S} w_i^S \Phi(\omega, \omega^S_i, \nu, \nu^S_i) u_{\theta}(t,x,\omega^{S}_i, \nu^{S}_i)\bigg) - f.
\end{equation}
Here, $k,\sigma,f$ are defined from \eqref{eq:genRTE} and $(\omega_i^S,\nu_i^S)$ are the Gauss-Legendre quadrature points, with quadrature weights $w_i$ of order $S$. 

We also need the following residuals for the initial and boundary conditions,
\begin{equation}
\label{eq:resb}
\begin{aligned}
    \res_{tb} &= \res_{tb,\theta}:= u_\theta - u_0, &\quad& \forall (x,\omega, \nu)\in D\times S\times \Lambda,\\
    \res_{sb} &= \res_{sb,\theta}:= u_\theta - u_b, &\quad& \forall (t, x,\omega, \nu)\in\Gamma_-.
\end{aligned}
\end{equation}

The strategy of PINNs, following \cite{KAR1,MM1}, is to minimize the \emph{residuals} \eqref{eq:res1} \eqref{eq:resb}, simultaneously over the admissible set of tuning parameters $\theta \in \Theta$ i.e 
\begin{equation}
    \label{eq:pinn1}
    {\rm Find}~\theta^{\ast} \in \Theta:\quad \theta^{\ast} = {\rm arg}\min\limits_{\theta \in \Theta} \left(
    \|\res_{int, \theta}\|^2_{L^2(\domain)} +
    \|\res_{sb,\theta}\|^2_{L^2(\Gamma_-)}+ 
    \|\res_{tb,\theta}\|^2_{L^2(D \times S \times \Lambda)}\right).
\end{equation}

However, the $L^2$ norms in \eqref{eq:pinn1} involve integrals that cannot be computed exactly and need to be approximated by suitable quadrature rules. It is exactly the place to recall the different training sets, introduced in section \ref{sec:22}. As these precisely correspond to the quadrature points of an underlying quadrature rule, we approximate the integrals in \eqref{eq:pinn1} with the corresponding quadrature rule to define the following loss function,
\begin{equation}
    \label{eq:lf1}
    J(\theta):= 
    \sum\limits_{j=1}^{N_{sb}} w_j^{sb} |\res_{sb,\theta}(z_j^{sb})|^{2}+ 
    \sum\limits_{j=1}^{N_{tb}} w_j^{tb} |\res_{tb,\theta}(z_j^{tb})|^{2} +
    \lambda\sum\limits_{j=1}^{N_{int}} w^{int}_j |\res_{int,\theta}(z_j^{int})|^{2}
\end{equation}
with the residuals $\res_{sb}, \res_{tb}$ and $\res_{int}$ defined in \eqref{eq:resb}, \eqref{eq:res1},
and $w^{sb},z^{sb}$, $w^{sb},z^{sb}$, $w^{int},z^{int}$ being  the quadrature weights and training points, defined in section \ref{sec:22}. Furthermore, $\lambda$ is a hyperparameter for balancing the residuals, on account of the PDE and the initial and boundary data, respectively.

It is common in machine learning \cite{DLbook} to regularize the minimization problem for the loss function i.e we seek to find,
\begin{equation}
\label{eq:lf2}
\theta^{\ast} = {\rm arg}\min\limits_{\theta \in \Theta} \left(J(\theta) + \lambda_{reg} J_{reg}(\theta) \right).
\end{equation}  
Here, $J_{reg}:\Theta \to \R$ is a \emph{weight regularization} (penalization) term. A popular choice is to set  $J_{reg}(\theta) = \|\theta_W\|^q_q$ for either $q=1$ (to induce sparsity) or $q=2$. The parameter $0 \leq \lambda_{reg} \ll 1$ balances the regularization term with the actual loss $J$ \eqref{eq:lf1}. 

The above minimization problem amounts to finding a minimum of a possibly non-convex function over a subset of $\R^M$ for possibly very large $M$. We will follow standard practice in machine learning and solving this minimization problem approximately by either (first-order) stochastic gradient descent methods such as ADAM \cite{adam} or even higher-order optimization methods such as different variants of the LBFGS algorithm \cite{lbfgs}. 

For notational simplicity, we denote the (approximate, local) minimum in \eqref{eq:lf2} as $\theta^{\ast}$ and the underlying deep neural network $u^{\ast}= u_{\theta^{\ast}}$ will be our physics-informed neural network (PINN) approximation for the solution $u$ of the PDE \eqref{eq:genRTE}. We summarize the PINN algorithm for approximating radiative transfer below,
\begin{algorithm} 
\label{alg:PINN} {Finding a physics informed neural network (PINN) to approximate the radiative intensity $u$ solving \eqref{eq:genRTE}}. 
\begin{itemize}
\item [{\bf Inputs}:] Underlying domain $\domain$, coefficients and data for the radiative transfer equation \eqref{eq:genRTE}, quadrature points and weights for underlying quadrature rules, non-convex gradient based optimization algorithms.
\item [{\bf Goal}:] Find PINN $u^{\ast}= u_{\theta^{\ast}}$ for approximating the solution $u$ of \eqref{eq:genRTE} .
\item [{\bf Step $1$}:] Choose the training sets as described in section \ref{sec:22}. 
\item [{\bf Step $2$}:] For an initial value of the weight vector $\overline{\theta} \in \Theta$, evaluate the neural network $u_{\overline{\theta}}$ \eqref{eq:ann1}, the PDE residual \eqref{eq:res1}, the boundary residuals \eqref{eq:resb}, the loss function \eqref{eq:lf2} and its gradients to initialize the underlying optimization
algorithm.
\item [{\bf Step $3$}:] Run the optimization algorithm till an approximate local minimum $\theta^{\ast}$ of \eqref{eq:lf2} is reached. The map $u^{\ast} = u_{\theta^{\ast}}$ is the desired PINN for approximating the solution $u$ of the radiative transfer equation.
\end{itemize}
\end{algorithm}
\subsection{Estimates on the generalization error}
\label{sec:25}
For the sake of definiteness and simplicity, we consider the spatial domain as $D = [0,1]^d$, with $d$ being the spatial dimension. Any rectangular domain $\prod\limits_{i=1}^d [a_i,b_i]$, with $a_i < b_i$, for any $a_i,b_i \in \R$ can be mapped to $[0,1]^d$ by rescaling. Similarly, logically (patch or block) cartesian domains can be transformed to $(0,1)^d$ by combinations of coordinate transforms. We also rescale time and frequency to set $T=1$ and $\Lambda = [0,1]$. Finally, the angular domains can be mapped onto to $[0,1]^{d-1}$ by rescaling the underlying polar coordinates. Hence, the underlying domain is $\dom = \domain = [0,1]^{2d+1}$. Thus, we can choose our interior training points $\train_{int}$, temporal boundary training points $\train_{tb}$ and spatial boundary training points $\train_{sb}$ as low-discrepancy Sobol points \cite{CAF1}. 

Our aim in this section is to derive a rigorous estimate on the so-called \emph{generalization error} (or approximation error) for the trained neural network $u^{\ast} = u_{\theta^{\ast}}$, which is the output of the PINNs algorithm \ref{alg:PINN}. This error is of the form,
\begin{equation}
    \label{eq:gerr}
\er_{G} = \er_{G}(\theta^{\ast}):=\left(\int\limits_{\dom}|u(t,x,\omega,\nu) - u^{\ast}(t,x,\omega,\nu)|^2 dz\right)^{\frac{1}{2}},
\end{equation}
with $dz = dx dt d\omega d\nu$ denoting the volume measure on $\dom$. 

We follow the recent paper \cite{MM1} and estimate the generalization error \eqref{eq:gerr}, in terms of \emph{training errors},
\begin{equation}
    \label{eq:etrain}
    \er_T^{sb}:= 
   \left(\sum\limits_{j=1}^{N_{sb}} w_j^{sb} |\res_{sb,\theta^{\ast}}(z_j^{sb})|^{2}\right)^{\frac{1}{2}}, \quad 
   \er_T^{tb}:=
    \left(\sum\limits_{j=1}^{N_{tb}} w_j^{tb} |\res_{tb,\theta^{\ast}}(z_j^{tb})|^{2}\right)^{\frac{1}{2}}, \quad 
    \er_T^{int}:= \left(\sum\limits_{j=1}^{N_{int}} w^{int}_j |\res_{int,\theta^{\ast}}(z_j^{int})|^{2}\right)^{\frac{1}{2}} 
\end{equation}
Note that the training errors, defined above, correspond to a local minimizer $\theta^{\ast}$ of \eqref{eq:lf2} and are readily computable from the loss function \eqref{eq:lf2}, during and at the end of the training process. 

The detailed estimate on the generalization error in Lemma \ref{lem:1}, together with the assumptions on the underlying coefficients, functions and neural network, is presented and proved in Appendix \ref{sec:a1}. We direct the interested reader to the appendix and focus on the following form of the error estimate \eqref{eq:egbd},
\begin{equation}
    \label{eq:egbd1}
    \begin{aligned}
    (\er_G)^2 &\leq C_1\left((\er_T^{tb})^2 + c(\er_T^{sb})^2 + c(\er_T^{int})^2\right) \\
    &+ C_2\left(\frac{(\log(N_{tb}))^{2d}}{N_{tb}} + c \frac{(\log(N_{sb}))^{2d}}{N_{sb}}+ c \frac{(\log(N_{int}))^{2d+1}}{N_{int}} + c N^{-2s}_{S}\right), 
    \end{aligned}
\end{equation}
with finite constants $C_1= C$, $C_2=CC^{\ast}$ defined in \eqref{eq:c}. The following remarks about the bound \eqref{eq:egbd1} are in order,
\begin{remark}
The estimate \eqref{eq:egbd1} bounds the generalization error in terms of the training errors defined in \eqref{eq:etrain} and the number of training points $N_{int,sb,tb}$ as well as quadrature points $N_{S}$ for approximating the scattering integral in \eqref{eq:genRTE}. Although we have no apriori estimate on the training errors, as argued in \cite{MM1}, these errors can readily calculated after the training process has completed. Thus, the estimate \eqref{eq:egbd1} tells us that under the assumptions that the constants appearing in \eqref{eq:egbd1} are finite, \emph{as long as the PINN is trained well, it generalizes well}. This is exactly in the spirit of generalization results in theoretical machine learning \cite{MLbook2}.
\end{remark}
\begin{remark}
We see from the right hand side of the bound \eqref{eq:egbd1} that the dimensional dependence of the upper bound is only a logarithmic factor. This is not a severe restriction in this case, as the spatial dimension $d$ is atmost $3$. It is well known \cite{CAF1} that the logarithmic factor in the rhs of \eqref{eq:egbd1} starts affecting the rate of decay only when $N_{int} < 2^{2d+1}$. Thus as long as $N_{int} > 128$ and $N_{tb},N_{tb} > 64$, we should see a linear decay in the error contributions of the Sobol points in \eqref{eq:egbd1}. Hence, we claim that as long as the training errors do not depend on the underlying dimension, the estimate \eqref{eq:egbd1} suggests that the PINNs algorithm \ref{alg:PINN} will \emph{not suffer from a curse of dimensionality}. 
\end{remark}
\begin{remark}
The estimate \eqref{eq:egbd1} brings out the role of the speed of light $c$ very clearly. As long as $c$ is finite, we can rescale time to set $c=1$. Nevertheless, the constant $C$ in \eqref{eq:egbd1} grows exponentially with the rescaled time, deteriorating the control on the error provided by the bound \eqref{eq:egbd1}. Thus, this bound is not suitable for steady-state problems (formally) obtained by letting $c \rightarrow \infty$. Nevertheless, a modified error estimate can be derived for the steady state case and we present it in the appendix \ref{sec:a2}.
\end{remark}
\section{Numerical Experiments}
\label{sec:3}
\subsection{Implementation}
The PINNs algorithm \ref{alg:PINN} has been implemented within the PyTorch framework \cite{torch} and the code can be downloaded from \url{https://github.com/mroberto166/PinnsSub}. As is well documented \cite{KAR1,KAR2,MM1}, the coding and implementation of PINNs is extremely simple, particularly when compared to standard methods such as finite elements. Only a few lines of Python code suffice for this purpose. All the numerical experiments were performed on a single GeForce GTX1080 GPU. 

The PINNs algorithm has the following hyperparameters, the number of hidden layers $K-1$, the width of each hidden layer $d_k\equiv \tilde{d}$ in \eqref{eq:ann1}, the specific activation function $A$, the parameter $\lambda$ in the loss function \eqref{eq:lf1}, the regularization parameter $\lambda_{reg}$ in the cumulative loss function \eqref{eq:lf2} and the specific gradient descent algorithm for approximating the optimization problem \eqref{eq:lf2}. We use the hyperbolic tangent $\tanh$ activation function, thus ensuring that all the smoothness hypothesis on the resulting neural networks, as required in lemmas \ref{lem:1} and \ref{lem:2} are satisfied. Moreover, we use the second-order LBFGS method \cite{lbfgs} as the optimizer. We follow the ensemble training procedure of \cite{LMR1} in order to choose the remaining hyperparameters. To this end, we consider a range of values, shown in Table \ref{tab:1}, for the number of hidden layers, the depth of each hidden layer, the parameter $\lambda$ and the regularization parameter  $\lambda_{reg}$. For each configuration in the ensemble, the resulting model is retrained (in parallel) $n_\theta$ times with different random starting values of the trainable weights in the optimization algorithm and the one yielding the smallest value of the training loss is selected.

\begin{table}[htbp] 
    \centering
    \renewcommand{\arraystretch}{1.1} 
    
    \footnotesize{
        \begin{tabular}{ c c c c c c c c} 
            \toprule
              & \bfseries $K-1$  & \bfseries $\tilde{d}$  &\bfseries $\lambda$ &\bfseries $\lambda_{reg}$ & \bfseries $n_\theta$\\ 
            \midrule
            \midrule
             Example \ref{sec:32}, \ref{sec:33}& 4, 8  & 16, 20, 24 & 0.1, 1, 10 &0 &5\\
              \midrule
              Example \ref{sec:34} &4, 8  & 16, 20 & 0.1, 1 & 0, $10^{-6}$, $10^{-5}$ &10\\
              \midrule
              Example \ref{sec:35}&4, 8, 12, 16, 20  & 16, 20, 24, 28, 32, 36, 40 & 0.1, 1 & 0 &20\\
              \midrule
              Example \ref{sec:42}&4, 8  & 16, 20, 24 & 1, 10 & 0 & 5\\
            \bottomrule
        \end{tabular}
    \caption{Hyperparameter configurations and number of retrainings employed in the ensemble training of PINNs for the radiative transfer equation \eqref{eq:genRTE}}
        \label{tab:1}
    }
\end{table}
\subsection{Monochromatic stationary radiative transfer in one space dimension}

\label{sec:32}
We begin with the much simpler case of steady state radiative transfer in the one space dimension, also referred to as slab geometry \cite{Fran}. In this case, the radiative transfer equations \eqref{eq:genRTE} simplify to,
\begin{equation}
\label{eq:1drad}
\begin{aligned}
\mu\frac{\partial }{\partial x}u(x,\mu)  &+\Big(\sigma(x) + k(x)\Big) u(x,\mu) = \frac{\sigma(x)}{2}\int_{-1}^{1}   \Phi(\mu, \mu')u(x,\mu')d\mu',\quad \mu=\cos(\theta), \quad  (x,\mu) \in [0,1]\times[-1,1].\\
\end{aligned}
\end{equation}
We follow the setup of \cite{PON} where the authors benchmarked least squares finite element methods for one-dimensional radiative transfer on this problem. As in \cite{PON}, the following \emph{inflow} boundary conditions are imposed:
\begin{equation}
\begin{aligned}
u(0,\mu)&=1, \quad \mu\in(0,1],\\
u(1,\mu)&=0, \quad \mu\in[-1,0).\\
\end{aligned}
\end{equation}
Note that the boundary conditions allow for possible discontinuities at $\mu=0$. The coefficients and scattering kernel are,
\begin{equation}
    \sigma(x)=x, \quad k(x)=0, \quad \Phi(\mu',\mu)= \sum_{\ell=0}^L d_\ell P_\ell(\mu)P_\ell(\mu'), \quad d_0=1,
\end{equation}
with $P_\ell(\mu)$  denoting the Legendre polynomial of order $\ell$. We employ the sequence of coefficients $d_\ell= \{1.0, 1.98398, 1.50823, 0.70075, 0.23489, 0.05133, 0.00760, 0.00048\}$, proposed in \cite{PON}. Although only in $2$ dimensions, this problem is nevertheless considered rather challenging on account of the possible presence of discontinuities. 
\begin{table}[htbp] 
    \centering
    \renewcommand{\arraystretch}{1.1} 
    
    \footnotesize{
        \begin{tabular}{ c c  c c  c c c c c c} 
            \toprule
              $N_{int}$  & $N_{sb}$   &\bfseries $K-1$ & \bfseries $\bar{d}$ &\bfseries $\lambda$&  $\er_T$ & $||u_- - u_-^{\ast}||_{L^2}$ & $||u_+ - u_+^{\ast}||_{L^2}$ &  $||u- u^{\ast}||_{L^2}$ &  Training Time    \\ 
            \midrule
            \midrule
            8192    & 2048  &8 &24& 0.1  & 0.00015 & 1.1$\%$ & 1.2$\%$ & 0.24$\%$& 20~min \\
            \bottomrule
        \end{tabular}
        \caption{Results for monochromatic stationary radiative transfer in one space dimension. }
        \label{tab:rad1d}
    }
\end{table}
We use the PINNs algorithm \ref{alg:PINN} to approximate \eqref{eq:1drad}, with Sobol points for the interior training set $\train_{int}$ and spatial boundary training set $\train_{sb}$. Similarly, a Gauss-Legendre quadrature rule with $N_S =10$ quadrature points is used for approximating the integral with the scattering kernel. We also set $N_{int}=8192$, $N_{sb} = 2048$, for this experiment. The hyperparameters that resulted from the ensemble training are presented in Table \ref{tab:rad1d}. As seen from the table, a very low training error is obtained in this case, together with a comparably low generalization error of $0.24\%$. A contour plot of the resulting radiative intensity in $(x,\mu)$-plane is presented in figure \ref{fig:cont_1d}. The results are very similar to those obtained with a finite element solver. It is interesting to note that this very good match with the finite element method is obtained with a training time of $20$ minutes on a CPU. 
\begin{figure}[ht]
        \centering
        \includegraphics[draft=False,width=1\linewidth]{{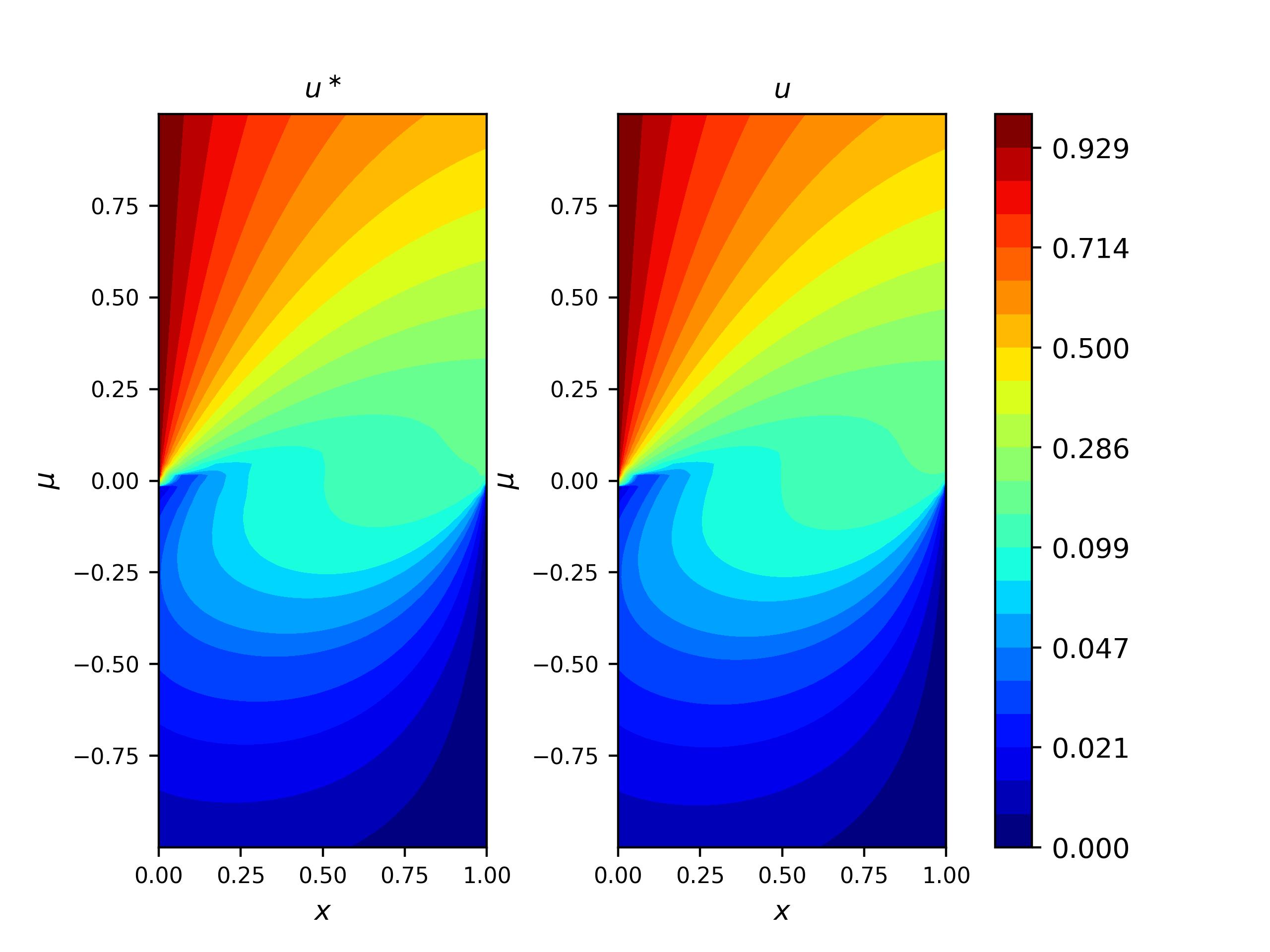}}
   
    \caption{Contour plot of the PINN radiative intensity $u^\ast(x,\mu)$ for the 1D monochromatic experiment (left), compared with the solution $u(x,\mu)$ obtained with a finite element solver (right)}
\label{fig:cont_1d}
\end{figure}

Another attractive feature of this simplified problem lies in the fact that the authors in \cite{CEN} obtained an exact analytical solution for it. Although it is very complicated to evaluate this solution for the whole $(x,\mu)$-plane, its values on the boundaries can be readily evaluated i.e. we can readily compute $u_{-}(\mu) =  u(0,\mu)$ and $u_+(\mu) = u(1,\mu)$. We do so and compare the exact solution with the trained PINN, denoted by $u^{\ast}_{\pm}$. These results are plotted in figure \ref{fig:bound_int}. We see from this figure that the PINN is able to very accurately approximate the discontinuous exact solution at the boundary. A quantitative comparison in performed by computed the errors $u_{\pm} - u^{\ast}_{\pm}$ in $L^2$-norm. These errors, presented in table \ref{tab:rad1d}, are very small for both boundaries and further demonstrate that the PINN is able to approximate the underlying discontinuous solution to high-accuracy, at very low computational cost. One can possibly further reduce the error by using adaptive activation functions as suggested in \cite{jag3}. However, we did not observe any further reduction in the already low training error by using adaptive activation functions.  
\begin{figure}[ht!]

    \begin{subfigure}{.49\textwidth}
        \centering
        \includegraphics[width=1\linewidth]{{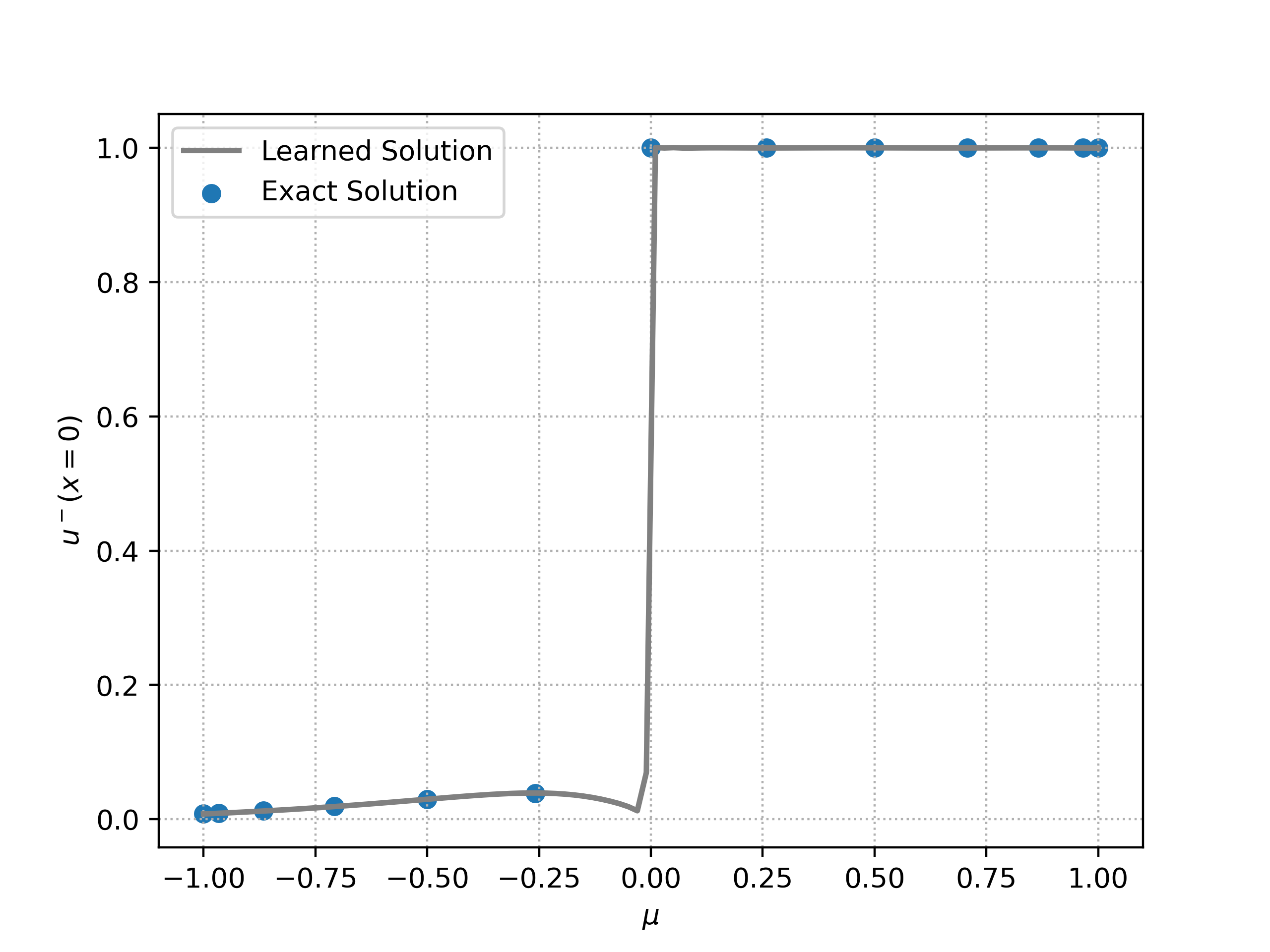}}
    \end{subfigure}
    \begin{subfigure}{.49\textwidth}
        \centering
        \includegraphics[width=1\linewidth]{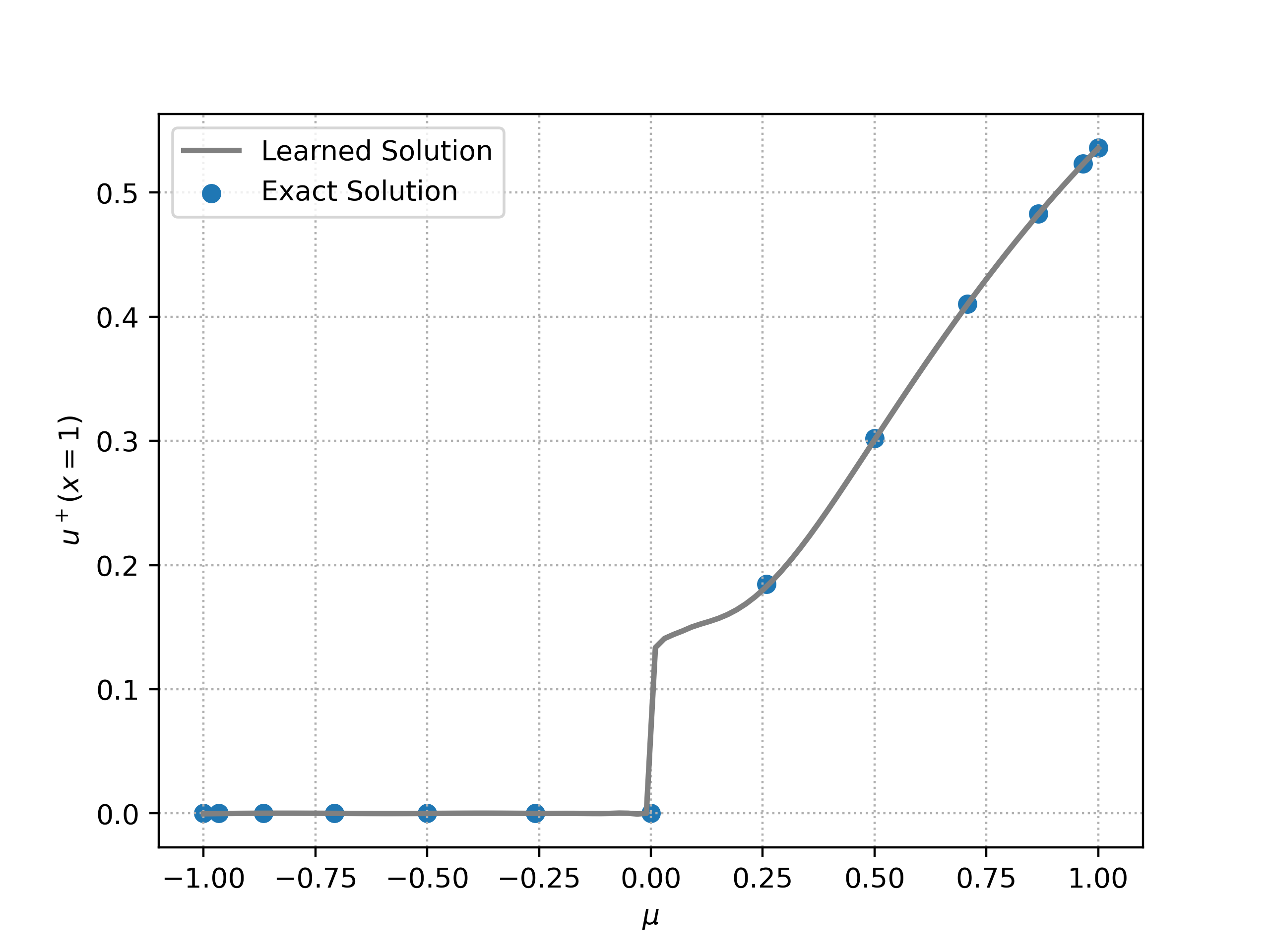}
    \end{subfigure}
    \caption{Comparison of the analytical and PINNN radiative  intensity at the physical domain boundaries for the stationary monochromatic radiative transfer in one-space dimension.}
\label{fig:bound_int}
\end{figure}
\subsection{Monochromatic stationary radiative transfer in three space dimensions} 
\label{sec:33}
Next, we consider a monochromatic and stationary version of the general radiative transfer equations \eqref{eq:genRTE}, but in three space dimensions. Already, this problem is in $5$ dimensions and is challenging on account of possibly high computational cost. We use the same setup as in \cite{GRELLA} (section 8.2, experiment 3) and consider the problem in the unit cube $D=[0,1]^3$ where a source, located at the center $c = (0.5, 0.5, 0.5)$, radiates into the surrounding medium.  We consider no further radiation entering the domain (zero Dirichlet boundary conditions). The source term $f$ is given by
\begin{equation}
    f(x) = k(x)I_b(x),\quad I_b(x) = \left\{
        \begin{array}{ll}
            0.5 -r, & \quad r \leq 0.5  \\
            0, & \quad \text{otherwise}
        \end{array}
    \right.
\end{equation}
with $r = |x -c|$. The absorption coefficient is $k(x)=I_b(x)$ and isotropic scattering $\Phi = 1$, with unit scattering coefficient $\sigma(x)=1$ is considered.

As before, we use Sobol points for the interior training set $\train_{int}$ and boundary training set $\train_{sb}$. Quadrature points, corresponding to a {Gauss} quadrature rule of order 20 are also used. We set $N_{int}=16384$, $N_{sb}=12288$ and $N_S=100$. The hyperparameters, corresponding to the best performing networks, that result from ensemble training are presented in Table \ref{tab:rad3d}. We see from this table that this hyperparameter configuration resulted in a very low (total) training error of $4.4 \times 10^{-4}$, which is comparable to those obtained in the one-space dimension case (see table \ref{tab:rad1d}). 
\begin{table}[htbp] 
    \centering
    \renewcommand{\arraystretch}{1.1} 
    
    \footnotesize{
        \begin{tabular}{ c c  c c  c  c c} 
            \toprule
              $N_{int}$  & $N_{sb}$   &\bfseries $K-1$ & \bfseries $\tilde{d}$ &\bfseries $\lambda$&  $\er_T$ & Training Time  \\ 
            \midrule
            \midrule
            16384    & 12288  &8 &24& 0.1   & 0.00044 & 1~hr~9~min \\
            \bottomrule
        \end{tabular}
        \caption{Results of the ensemble training for the stationary monochromatic radiative transfer in three space dimensions.}
        \label{tab:rad3d}
    }
\end{table}
\par As there is no analytical solution available for the radiative intensity in this case, we cannot compute generalization errors. However, based on the theory (see estimate \eqref{eq:egbd}) and on the comparison with the one-dimensional case, we expect very low generalization errors when the training errors are this low. Moreover, we can perform qualitative comparisons with the results obtained in \cite{GRELLA} with an efficient discrete ordinate method. To this end, we plot three-dimensional volume plot for the \emph{incident radiation} $G(x)$ (see the first equation in \eqref{eq:inc_rad} for definition) in figure \ref{fig:cont_3d}. We see from this figure that the results with PINN are very similar to the results with the discrete ordinate method, shown in \cite{GRELLA} (figure 8.12, page 126). Thus, we are able to approximate the incident radiation to the same accuracy as a discrete ordinate method. The main differences lies in the simplicity of implementation and very low computational cost. We observe from table \ref{tab:rad3d} that the PINN was trained in approximately $70$ minutes on a single GPU. This should be contrasted with the very intricate parallel algorithm of \cite{GRELLA}, which required considerably more computational time as the method resulted in very number of degrees of freedom ranging from $200000-600000$. 
\begin{figure}[ht!]
        \centering
        \includegraphics[width=0.7\linewidth]{{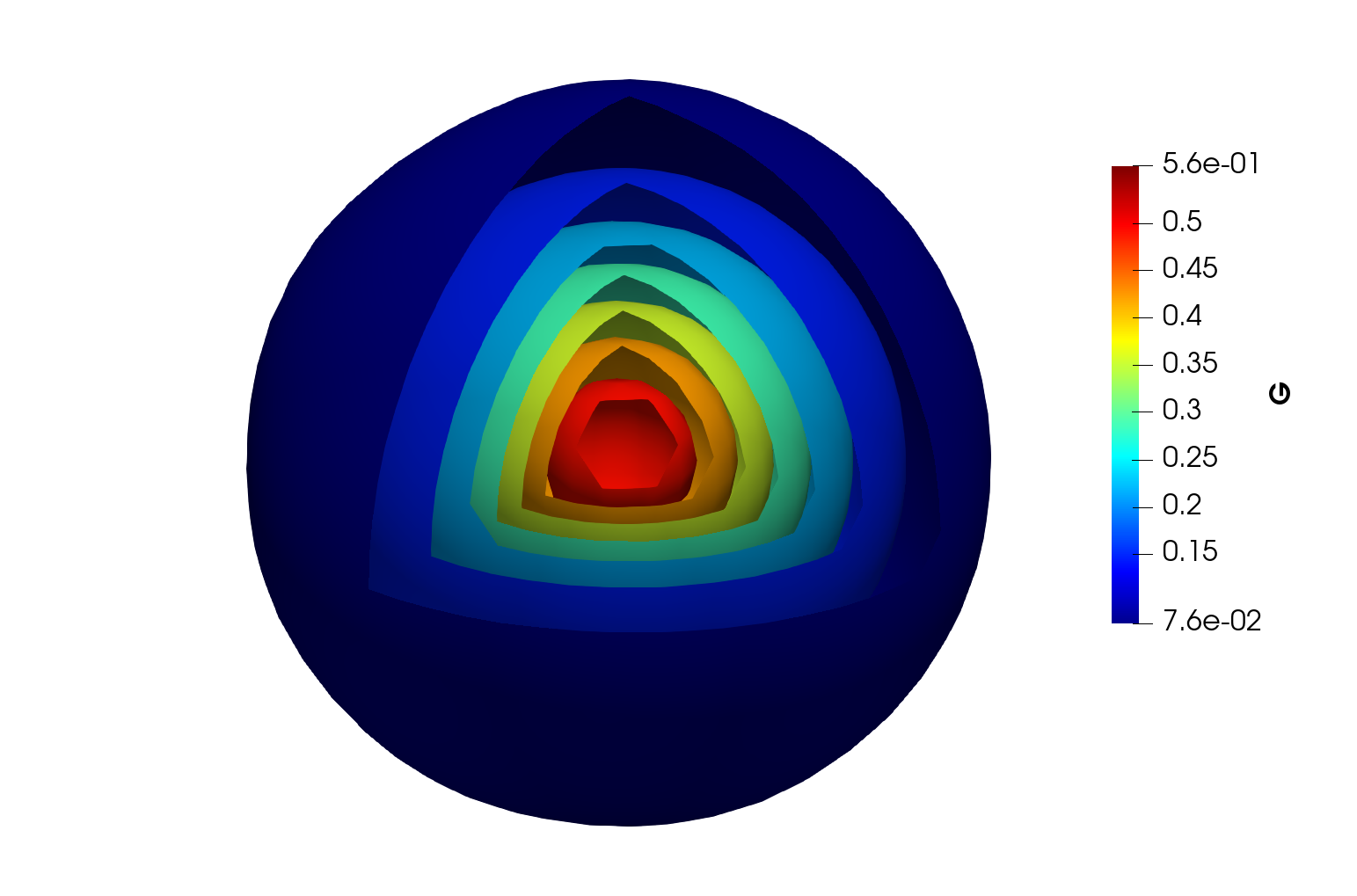}}
    \caption{Contour plot of the incident radiation $G(x)$ for the 3D monochromatic experiment}
\label{fig:cont_3d}
\end{figure}
\subsection{Polychromatic stationary radiative transfer in three space dimensions}
\label{sec:34}
Next, we consider the most general case of the steady state radiative transfer equation \eqref{eq:sRTE} by following the setup of \cite{RAN1} and references therein, where \eqref{eq:sRTE} is considered in the unit cube $D=[0,1]^3$ and in the frequency domain $\Lambda=[-6,6]$, with normalization of energy groups. Furthermore,  we consider a simple case of zero absorption, isotropic kernel, zero Dirichlet boundary conditions and spherical symmetry. Under the assumptions, by integrating equation \eqref{eq:sRTE} over the unit sphere $S$, we arrive at the following ordinary differential equation for the radial flux i.e the incident heat flux \eqref{eq:inc_rad} along the radius, 
\begin{equation}
    \nabla\cdot F_r = \frac{1}{r^2}\frac{d}{dr}r^2F_r = 4\pi f(r,\nu)
\end{equation}
with $r= |x-(0.5,0.5,0.5)|$ (see also \cite{RAN1} and references therein).

An exact solution for the above ODE can be easily obtained. In particular, with the source term:
\begin{equation}
    f(x,\nu)= 
     \begin{cases}
       \sqrt{\pi}\phi(\nu)\Big(1 - 2r\Big) &\quad\text{if } r\leq 0.5, \\
       0 &\quad\text{otherwise},\\
     \end{cases} \quad\phi(\nu) = \frac{1}{\sqrt{\pi}}\exp{\big(-\nu^2\big)},
\end{equation}
the radial flux $F_r$ results in
\begin{equation}
\label{eq:rf}
    F_r = 
     \begin{cases}
       4\sqrt{\pi^3}\phi(\nu)\Big( \frac{r}{3} -\frac{r^2}{2}\Big) &\quad\text{if } r\leq 0.5, \\
        4\sqrt{\pi^3}\phi(\nu)\frac{1}{96r^2} &\quad\text{otherwise}.\\
     \end{cases}
\end{equation}
As in the previous numerical experiment, we use Sobol points for the interior and boundary training sets and Gauss quadrature points for integrating the scattering kernel, with $N_{int} = 16384$, $N_{sb}=12288$ and $N_S = 100$. The hyperparameters used in the ensemble trainig are reported in table \ref{tab:1} and the resulting best performing configuration is shown in table \ref{tab:flux_res}. We observe from this table that the resulting training error is $1.6 \times 10^{-3}$, which is about three times higher than the training error with the monochromatic experiment (see table \ref{tab:rad3d}). This is not surprising as the underlying problem is more complicated on account of introducing frequency as an additional variable and resulting in a $6$-dimensional problem. 

As no analytical solution is available for the radiative intensity, we cannot compute the generalization error \eqref{eq:gerrs}. However, we can compute the error between the analytical radial flux \eqref{eq:rf} and the PINN approximation (computed from the intensity with a Gauss-Legendre quadrature rule). We show the resulting $L^2$-norm of the error in table \ref{tab:flux_res}. We see from this table that the error for the flux is quite low at approximately $2\%$ relative error, even for this rather complicated underlying problem. Moreover, the training time is only one hour. Thus, PINNs are able to approximate the underlying solution to high accuracy at low computational cost for this $6$-dimensional problem. 
\begin{table}[htbp] 
    \centering
    \renewcommand{\arraystretch}{1.1} 
    
    \footnotesize{
        \begin{tabular}{ c c   c c c c  c c } 
            \toprule
              $N_{int}$  & $N_{sb}$   &\bfseries $K-1$ & \bfseries $\tilde{d}$  &\bfseries $\lambda$&  $\er_T$ & $||F_r - F_r^\ast||_{L^2}$ & Training Time    \\ 
            \midrule
            \midrule
            16384    & 12288 &8 &20&0.1   & 0.0016& 2.1 $\%$ & 1~hr~6~min\\
            \bottomrule
        \end{tabular}
        \caption{Results for steady polychromatic radiative transfer in three space dimensions}.
        \label{tab:flux_res}
    }
\end{table}
\subsection{Polychromatic time-dependent Radiative transfer in three space dimensions}
\label{sec:35}
For the final numerical experiment, we consider the configuration proposed in \cite{Graz}, which is widely used in benchmarking the radiative transport modules in production codes for radiation-(magneto)hydrodynamics, in the context of Astrophysics \cite{Bell}. The setup is as follows;  a sphere with radius $R_i$ and fixed temperature $T_S$ is surrounded by a cold static medium at temperature $T_m< T_S$. The experiment might represent, for instance, the model of a star radiating in the surrounding atmosphere. It is assumed that the sphere, as well as the surrounding medium, are emitting with a Planckian distribution 
\begin{equation}
    B(T, \nu) = \frac{2h\nu^3}{c^2}\frac{1}{e^{\frac{h\nu}{k_b T}} -1}
\end{equation}
with $h$ and $k_b$ being the Planck and Boltzmann constant, and $c$ the speed of light.

To make the problem tractable, the authors of \cite{Graz} neglect scattering entirely by setting $\sigma \equiv 0$. Moreover, the absorption coefficient is modeled by $k(x,\nu) = k_{\nu}$, with $\nu$ being the frequency. The emission term is modeled by $f(x,\nu) = k_{\nu}B(T_m,\nu)$, resulting in the following form of the radiative transfer equation \eqref{eq:genRTE},
\begin{equation}
\label{eq:RTE_time_red}
\begin{aligned}
 \frac{1}{c}\frac{\partial u}{\partial t}  + \omega\cdot\nabla_x u  = k_\nu (B(T_m, \nu) - u),\quad (t,x,n,\nu)\in D_T \times S\times\Lambda.
\end{aligned}
\end{equation}

In the context of radiation-(magneto)hydrodynamics, one is mostly interested in the angular moments of the radiative intensity that naturally arise in calculating the contridbution of radiation to the total energy of the fluid (plasma). Hence, it is customary to integrate \eqref{eq:RTE_time_red} over the sphere $S$ to derive the following PDE for incident radiation \eqref{eq:inc_rad}:
\begin{equation}
\label{eq:RTE_G}
\begin{aligned}
 \frac{1}{c}\frac{\partial }{\partial t}G  + \nabla_x \cdot F  = k_\nu\Big(b(T_m, \nu) - G\Big),\quad (t,x,\nu)\in D_T \times\Lambda.
\end{aligned}
\end{equation}
with $b(T, \nu)= 4\pi B(T, \nu)$.

However, the PDE \eqref{eq:RTE_G} is not closed and one needs a closure for the flux $F$ in terms of the incident radiation $G$. It is common practice in astrophysics to use the so-called \emph{diffusion approximation} of the flux
\cite{CAST}:
\begin{equation}
\label{eq:diff_approx}
 F(t, x,\nu) = -\frac{1}{3k_\nu}\nabla  G(t, x, \nu),
\end{equation}
resulting in the following PDE,
\begin{equation}
\label{eq:RTE_G_approx}
\begin{aligned}
 \frac{1}{c}\frac{\partial }{\partial t}G  -\frac{1}{3k_\nu}\Delta  G  = k_\nu\Big(b(T_m, \nu) - G\Big), \quad (t,x,\nu)\in D_T \times\Lambda.
\end{aligned}
\end{equation}
Defining the \emph{Knudsen number} $K= Lk_{\nu}$ (with $L$
being a characteristic length scale), it is well known that the diffusion approximation is justified in the limit of $K \rightarrow \infty$. 

\begin{table}[htbp] 
    \centering
    \renewcommand{\arraystretch}{1.1} 
    
    \footnotesize{
        \begin{tabular}{ c c  c c c c c  c c c} 
            \toprule
                $k_\nu$&$N_{int}$  & $N_{sb}$ & $N_{tb}$  &\bfseries $K-1$ & \bfseries $\tilde{d}$  &\bfseries $\lambda$&  $\er_T$ & Training Time   \\ 
            \midrule
            \midrule
            1&16384    & 12288 & 12288 &4 &40&0.1   & 0.0028 & 3~hr~25~min\\
             \midrule
            10 &16384    & 12288 & 12288 &4 &40&0.1   & 0.012 & 2~hr~15~min\\
            \bottomrule
        \end{tabular}
        \caption{Results for polychromatic time-dependent radiative transfer in three space dimensions.}
        \label{tab:ast_res}
    }
    \end{table}

Although the PDE \eqref{eq:RTE_G_approx} is simpler than the full radiative transfer equation \eqref{eq:genRTE}, efficient numerical approximation of \eqref{eq:RTE_G_approx} is still quite challenging as the incident radiation is a function of $5$ variables. As it happens, PINNs provide an efficient method for approximating high-dimensional parabolic equations such as \eqref{eq:RTE_G_approx}, see \cite{MM1} section 3 for details.

However, by assuming radial symmetry and with the flux approximation given in \eqref{eq:diff_approx}, the differential equation \eqref{eq:RTE_G_approx}  admits an analytical solution satisfying the  initial and boundary conditions 
\begin{equation}
\begin{aligned}
    G(0,r,\nu) = b(T_m, \nu), \\
    G(t, r\rightarrow \infty,\nu) = b(T_m, \nu), \\
    G(0,R_i,\nu) = b(T_s, \nu). \\
\end{aligned}
\end{equation}
The exact solution for \eqref{eq:RTE_G_approx} then reads \cite{Graz}, 
 \begin{equation}
 \label{eq:ex}
 \begin{aligned}
  G(t,r,\nu) =& b(T_m, \nu) + \frac{R_i}{r}\Big( b(T_s, \nu) - b(T_m, \nu)\Big)F(t, r, \nu),\\
  F(t,r,\nu) =& \frac{1}{2}\exp{(-3k_\nu(r-R))}\Bigg\{  Erfc\bigg( \sqrt{\frac{3k_\nu}{4ct}}(r- R) - \sqrt{k_\nu ct} \bigg) +  Erfc\bigg( \sqrt{\frac{3k_\nu}{4ct}}(r- R) + \sqrt{k_\nu ct} \bigg) \Bigg\} .\\
 \end{aligned}
 \end{equation}
 
 For this numerical experiment, we will approximate the full time-dependent radiative transfer equations \eqref{eq:RTE_time_red} with the PINNs algorithm \ref{alg:PINN}. To this end, we consider \eqref{eq:RTE_time_red} in the spatial domain $D$ enclosed between two spheres with radii $R_i=2$ and $R_e=4$. Moreover, we introduce an auxiliary temporal variable $\tau = ct$ to rescale time to $[0,1]$, whereas the energy group ranges between $10^{15}$ and $10^{18}$.  We set $T_s= 150 eV$ and $T_m=120eV$.
 
 \begin{figure}[ht!]
 \begin{subfigure}{.48\textwidth}
        \centering
        \includegraphics[width=1\linewidth]{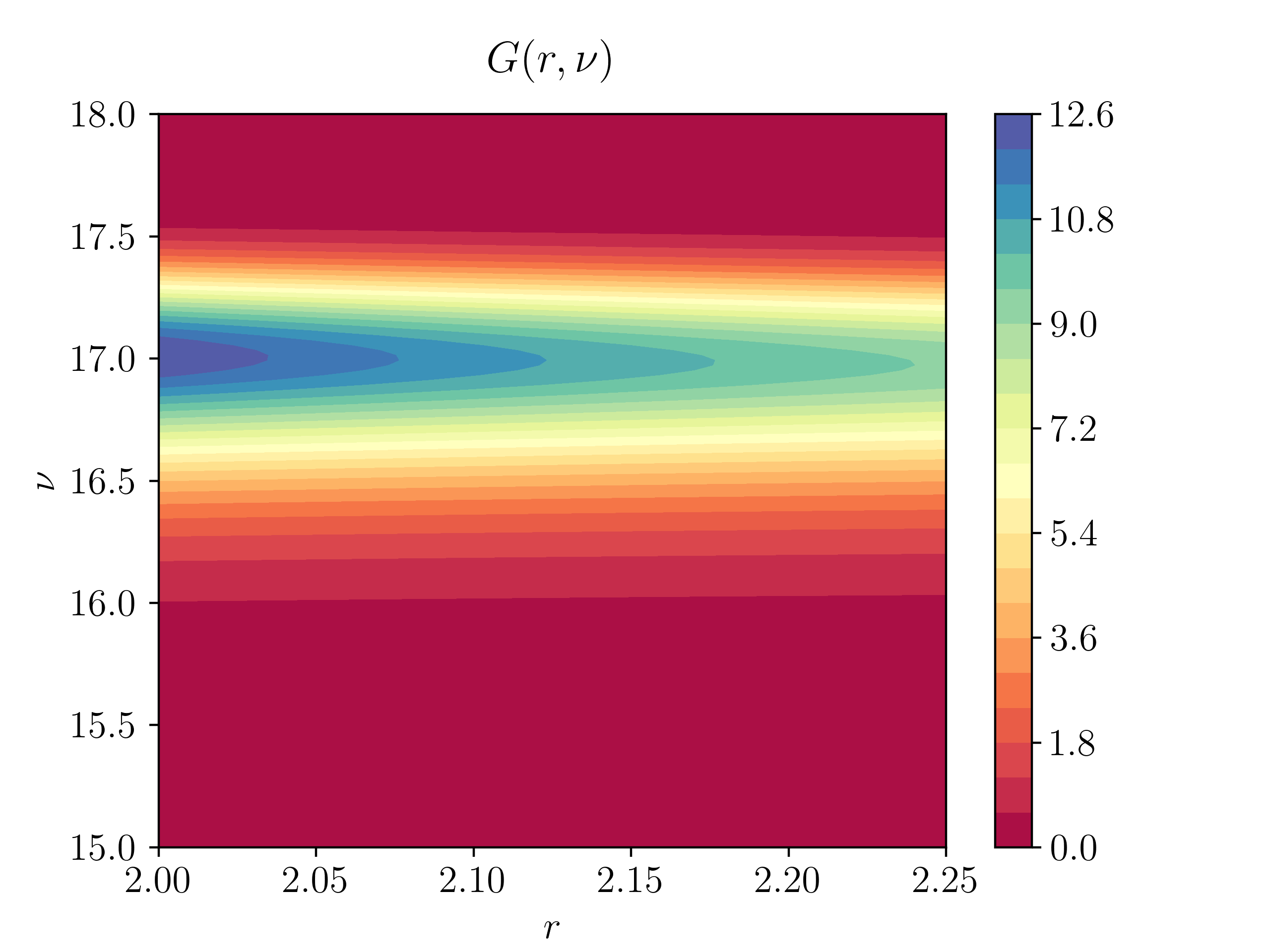}
        \caption{Exact solution of \eqref{eq:RTE_G_approx} for $k_\nu=1$}
    \end{subfigure}
    \begin{subfigure}{.48\textwidth}
        \centering
        \includegraphics[width=1\linewidth]{{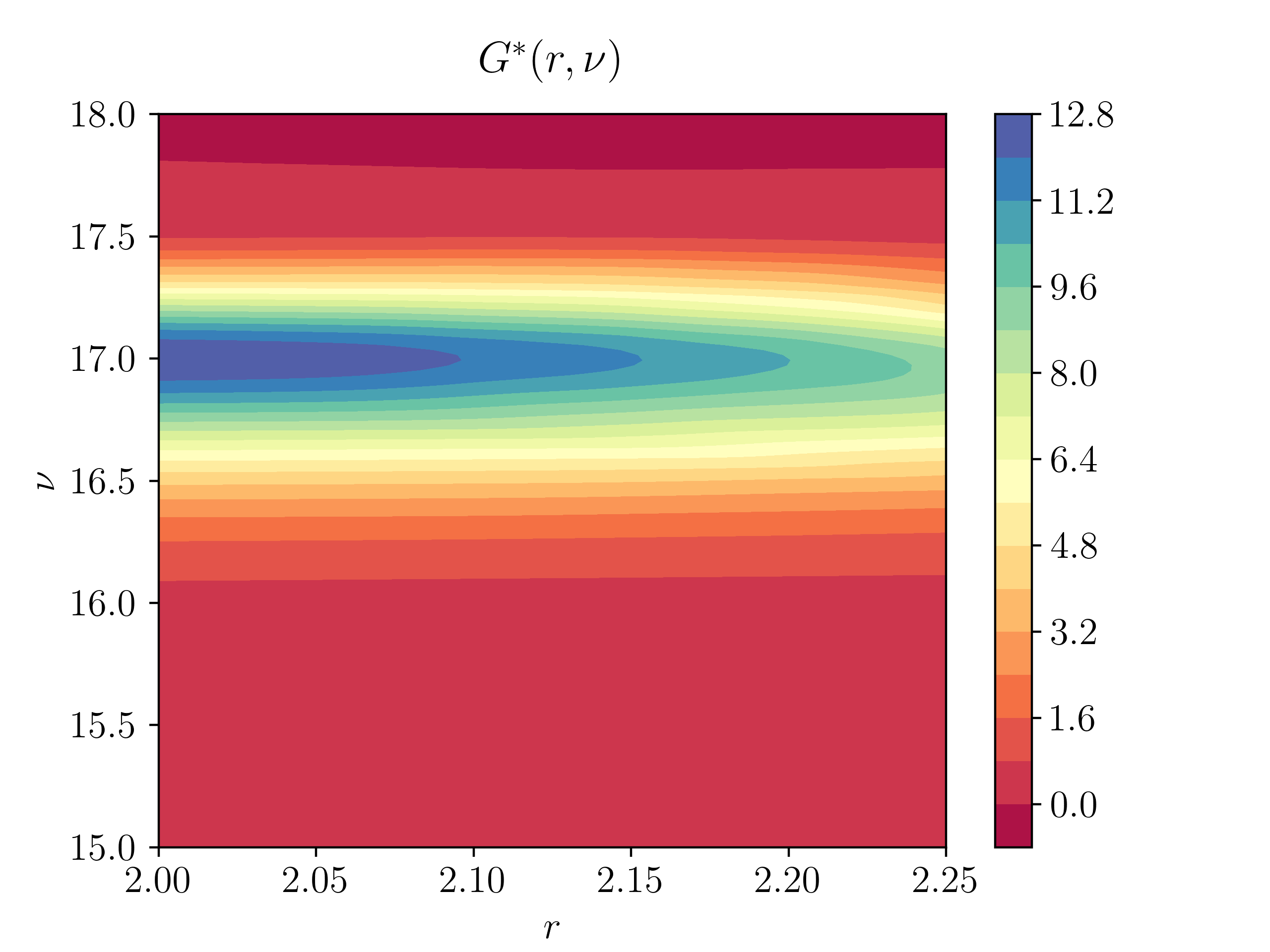}}
        \caption{PINN for $k_\nu=1$}
    \end{subfigure}
  
   \begin{subfigure}{.48\textwidth}
        \centering
        \includegraphics[width=1\linewidth]{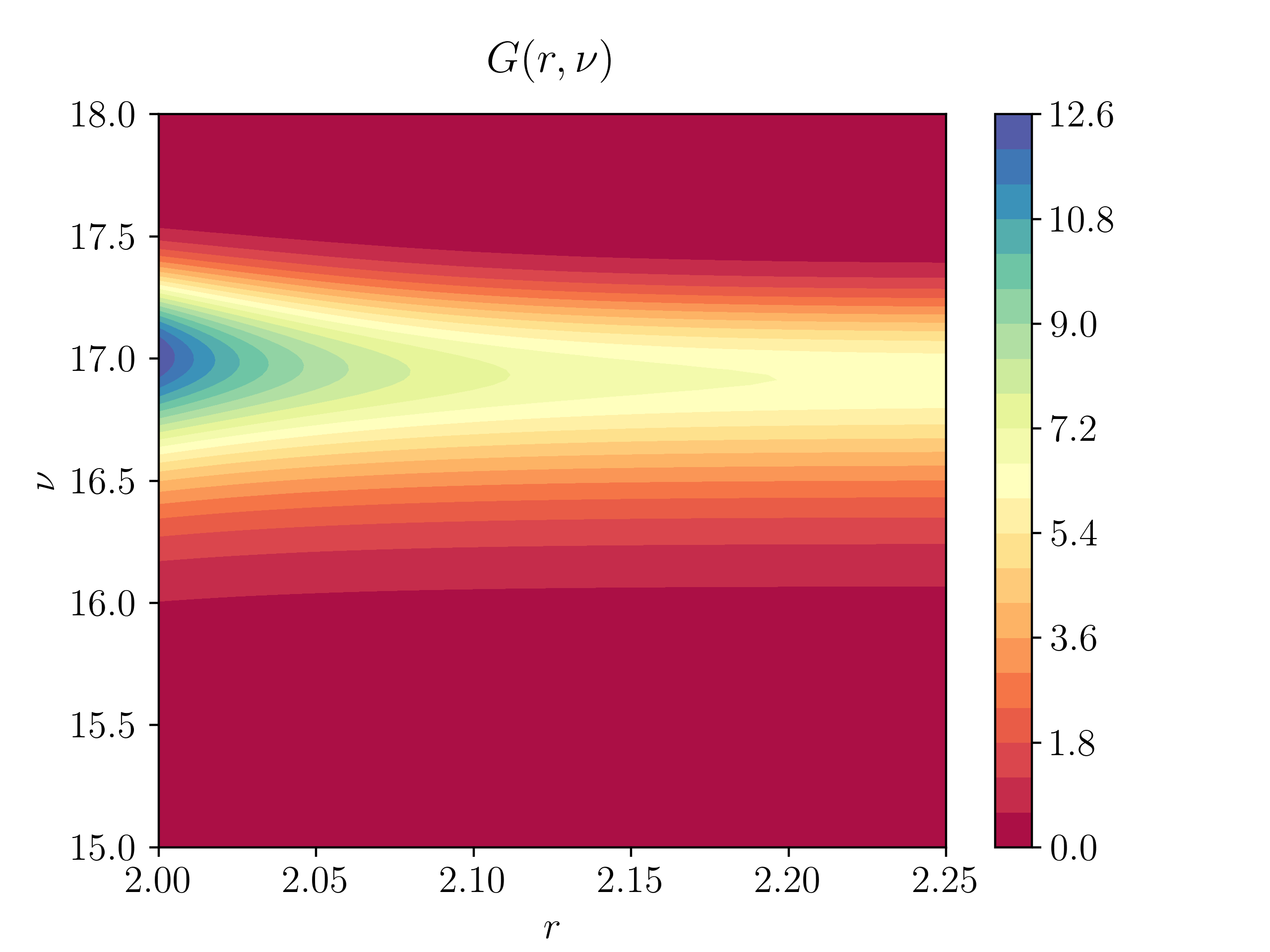}
        \caption{Exact solution of \eqref{eq:RTE_G_approx} for $k_\nu=10$}
    \end{subfigure}
      \begin{subfigure}{.48\textwidth}
        \centering
        \includegraphics[width=1\linewidth]{{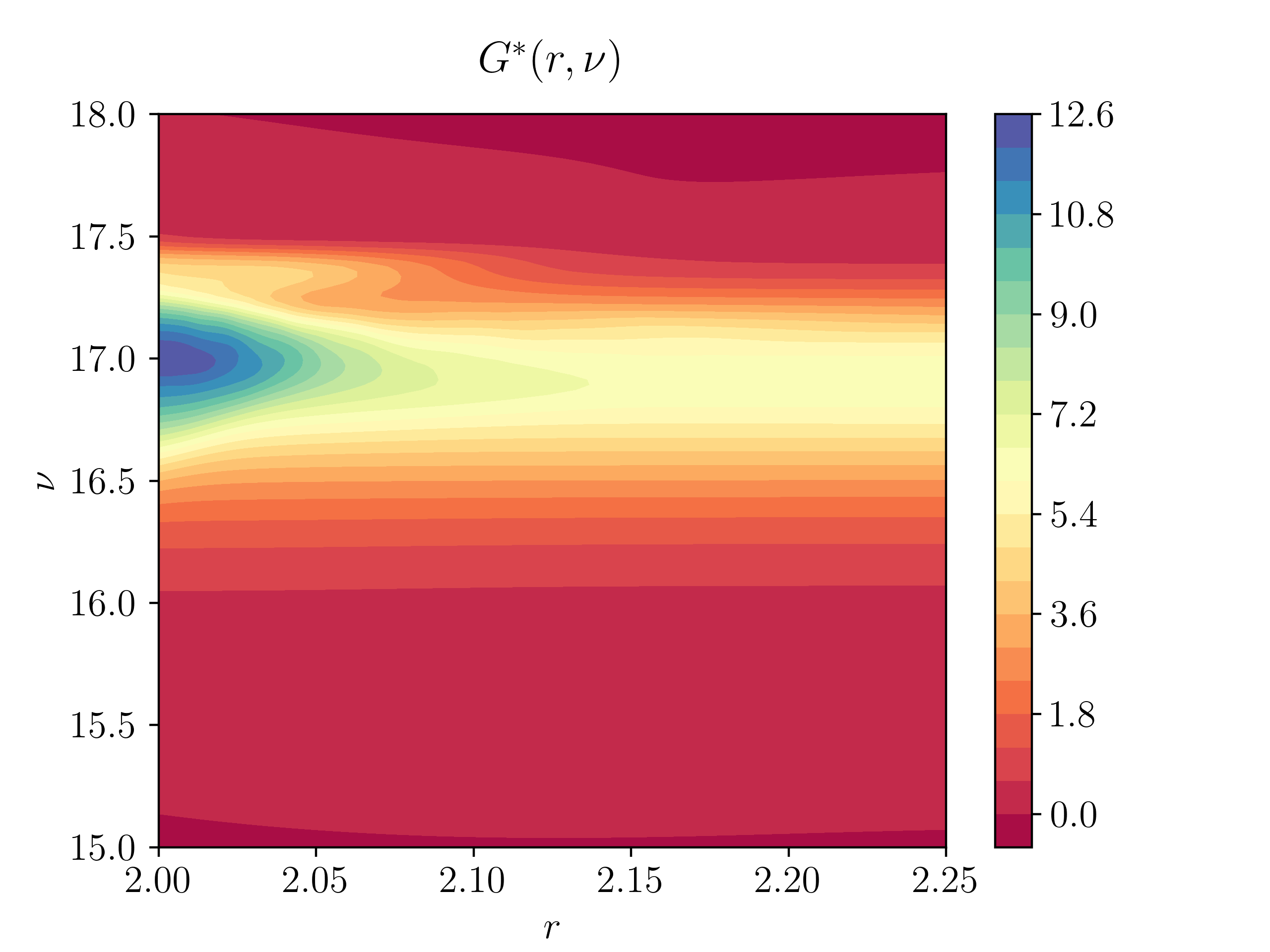}}
        \caption{PINN at $\tau=1$ for $k_\nu=10$}
    \end{subfigure}
    \caption{Comparison of incident radiationn with respect to the exact solution of the diffusion approximation \eqref{eq:RTE_G_approx} and PINN approximation of the full radiative transfer equation \eqref{eq:RTE_time_red} at rescaled time $\tau = 1$ for two different values of the absorption coefficient $k_\nu=1,10$ }
\label{fig:cont_ast}
\end{figure}
 
 The PINNs algorithm \ref{alg:PINN} employs Sobol points in the interior, spatial and temporal boundary training sets and we set $N_{int}=16384$, $N_{sb}=N_{tb}=12228$. Moreover, we solve this problem for two different values of the (constant in frequency) absorption coefficient i.e $k_{\nu}=1$ and $k_{\nu}=10$, resulting in two different Knudsen numbers of $K=2$ and $K=20$, respectively. Given the challenging nature of this problem, we choose slightly different ranges of the hyperparameters, presented in tabel \ref{tab:1} for ensemble training and also use $20$ retrainings, corresponding to different random starting values for the weights and biases in the training procedure. The resulting best performing configurations are reported in table \ref{tab:ast_res}. We observe from this table that PINNs provide a very low training error of $2.8 \times 10^{-3}$, for the $K=2$ case. This training error is comparable to the training errors for the previous two examples. The training error increases by a factor of $4$ for the $K=20$ case, but still remains relatively low. 
 
 As we do not have exact analytical formulas for the full radiative intensity, it is not possible to compute generalization errors. However to ascertain the quality of the solution, we compare with the exact solution \eqref{eq:ex} of the diffusion equation \eqref{eq:RTE_G_approx} for the incident radiation. This comparison is shown as contour plots for the incident radiation in the $(r,\nu)$-plane (with $r$ denoting the radial direction) in figure \ref{fig:cont_ast} as well as one-dimensional cross-sections for different values of the radius $r$ in figure \ref{fig:plot_ast}. As seen from both these figures, there is good agreement between the incident radiation, computed by a Gauss quadrature of the PINN approximation to the radiative intensity in \eqref{eq:RTE_time_red}, and the analytical solution of the diffusion approximation \eqref{eq:RTE_G_approx} for the $K=20$ case. This is not unexpected as the diffusion approximation is accurate for large Knudsen numbers. On the other hand, there is a significant difference between the the incident radiation, computed by a Gauss quadrature of the PINN approximation to the radiative intensity in \eqref{eq:RTE_time_red}, and the analytical solution of the diffusion approximation \eqref{eq:RTE_G_approx} for the $K=2$ case. This follows from the fact that the diffusion approximation will provide a poor approximation of \eqref{eq:RTE_time_red} for low Knudsen numbers. On the other hand, given the relatively low training error as well as the error estimate \eqref{eq:egbd}, coupled with the results of the previous numerical experiments, we argue that the PINN provides a much more accurate approximation to the underlying radiative intensity (and its moments) than the diffusion approximation will do, atleast for low to moderate Knudsen numbers. Hence, PINNs provide a viable and accurate method for competing radiative transfer in media with different optical properties. Moreover, the runtime for even this very complicated problem was reasonably small, ranging from two to three and half hours on a single GPU.
\begin{figure}[ht!]
 \begin{subfigure}{.49\textwidth}
        \centering
        \includegraphics[width=1\linewidth]{{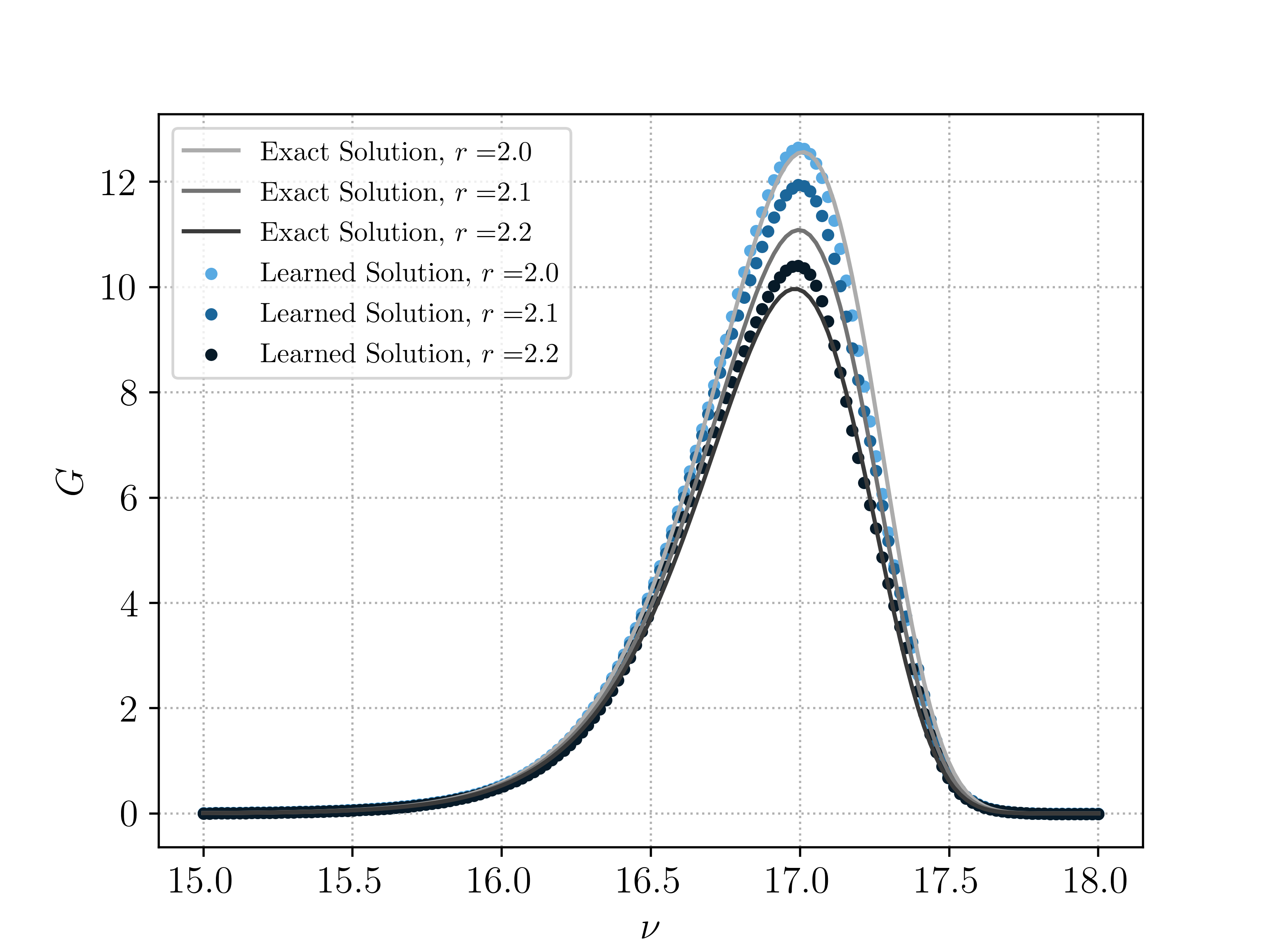}}
        \caption{Solution of \eqref{eq:RTE_time_red} and \eqref{eq:RTE_G} for $k_\nu=1$}
    \end{subfigure}
    \begin{subfigure}{.49\textwidth}
        \centering
        \includegraphics[width=1\linewidth]{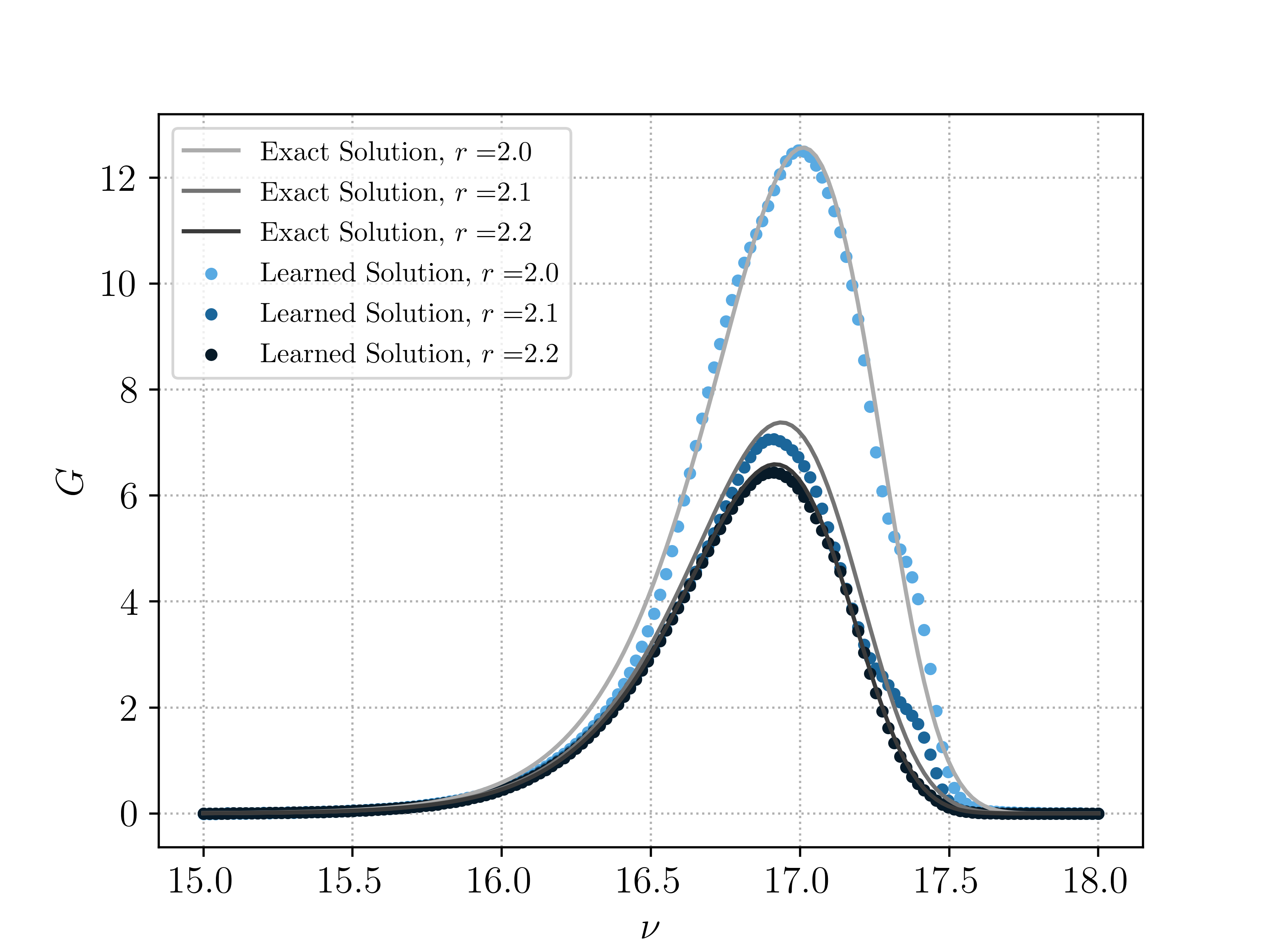}
        \caption{Solution of \eqref{eq:RTE_time_red}  and \eqref{eq:RTE_G} for $k_\nu=10$}
    \end{subfigure}
    \caption{Comparison of exact solutions of the diffusion approximation \eqref{eq:RTE_G_approx} with the PINN approximation of the full radiative transfer equation \eqref{eq:RTE_time_red} for two different values of the absorption coefficient $k_\nu=1,10$ at different radial locations and at rescaled time $\tau=1$ }
\label{fig:plot_ast}
\end{figure}
 \section{PINNs for the Inverse problem for radiative transfer}
 \label{sec:4}
 One of the most notable features of PINNs is their ability to approximate solutions of \emph{inverse problems}, with the same accuracy and computational cost as that of forward problems for PDEs. Moreover, the code for inverse problems ends up being a very minor modification to the code for forward problems, which makes PINNs extremely attractive for various applications \cite{KAR2,KAR4} and references therein. 
 
Here, we focus on the following inverse problem for radiative transfer. We consider the full time-dependent version on the radiative transfer equation \eqref{eq:genRTE}, with initial and boundary conditions \eqref{eq:bc}. The inverse problem is to compute unknown absorption coefficients $k$, scattering coefficients $\sigma$, scattering kernel $\Phi$ or emission term $f$, given measurements of either the full radiative intensity $u$ or its angular moments, such as the incident radiation $G$ or heat flux $F$ \eqref{eq:inc_rad}. For simplicity of exposition, we choose the following concrete inverse problem;\\
\emph{Given measurements of the incident radiation $G(t,x,\nu)$, find the unknown absorption coefficient $k = k(x,\nu)$ and the resulting radiative intensity $u(t,x,\omega,\nu)$ which solves the radiative transfer equation \eqref{eq:genRTE} with initial and boundary conditions \eqref{eq:bc}}. 
\par Other combinations of the measured and unknown quantities can be similarly considered. Clearly this inverse problem is ill-posed as multiple absorption coefficients might lead to the same incident radiation. However, we will aim to obtain one of the possible absorption coefficients, consistent with the measured incident radiation. 

To this end, we slightly modify the PINNs algorithm as described below,
\subsection{PINNs for the inverse problem}
Following \cite{KAR2,MM2}, we seek to find  the deep neural networks $k_{\theta_k}:D \times \Lambda \mapsto \R_+$ and $u_{\theta_{u}}: D_{T} \times S \times \Lambda \mapsto \R$, with the concatenated parameter vector $\theta = \{\theta_k,\theta_u\} \in \Theta$, approximating the absorption coefficient and radiative intensity, respectively. 

In addition to the interior training set $\train_{int}$, spatial boundary training set $\train_{sb}$ and temporal boundary training set $\train_{tb}$, defined in section \ref{sec:2}, we also required the so-called \emph{data training set} $\train_{d} = \{y^d_j\}$, for $1 \leq j \leq N_d$, and $y^d_j \in D_T \times \Lambda$. 

The residuals for initial and boundary conditions are given by $\res_{tb},\res_{sb}$ \eqref{eq:resb}. We slightly modify the PDE residual \eqref{eq:res1} to,
\begin{equation}
    \label{eq:res1i}
    \overline{\res}_{int,\theta}:= \frac{1}{c}\partial_t u_{\theta_u} +   \omega\cdot\nabla_x u_{\theta_u}  + k_{\theta_k} u_{\theta_u} + \sigma\Bigg(u_{\theta_u} - \frac{1}{s_d}
    \sum\limits_{i=1}^{N_S} w_i^S \Phi(\omega, \omega^S_i, \nu, \nu^S_i) u_{\theta_u}(t,x,\omega^{S}_i, \nu^{S}_i)\bigg) - f.
\end{equation}
We also need the \emph{data residual},
\begin{equation}
 \label{eq:resd}
 \res_{d,\theta}:= G\left(u_{\theta_u}\right) - \bar{G}(t,x,\nu), \quad \forall (t,x,\nu) \in D_T \times \Lambda,
\end{equation}
with $G$ being the incident radiation calculated from \eqref{eq:inc_rad} with a Gauss quadrature approximation of the angular integral and $\bar{G}$ being the measured incident radiation.

The resulting loss function is,
\begin{equation}
    \label{eq:lf3}
    J(\theta):= \sum\limits_{j=1}^{N_{d}} w_j^{d} |\res_{d,\theta}(y_j^{d})|^{2}+
    \sum\limits_{j=1}^{N_{sb}} w_j^{sb} |\res_{sb,\theta}(z_j^{sb})|^{2}+ 
    \sum\limits_{j=1}^{N_{tb}} w_j^{tb} |\res_{tb,\theta}(z_j^{tb})|^{2} +
    \lambda\sum\limits_{j=1}^{N_{int}} w^{int}_j |\overline{\res}_{int,\theta}(z_j^{int})|^{2}
\end{equation}
with the residuals $\res_{d},\res_{sb}, \res_{tb}$ and $\overline{\res}_{int}$ defined in \eqref{eq:resd}, \eqref{eq:resb} ,\eqref{eq:res1},
and $w^{d}, y^d, w^{sb},z^{sb}$, $w^{sb},z^{sb}$, $w^{int},z^{int}$ being  the quadrature weights and training points, corresponding to the data, boundary and interior training sets. 

The PINNs algorithm for the inverse problem is summarized as, 
\begin{algorithm} 
\label{alg:PINNi} {Finding physics informed neural network (PINNs) to approximate the absorption coefficient $k$ and radiative intensity $u$ solving the radiative transfer equation \eqref{eq:genRTE}, and consistent with measured data $\bar{G}$ for the incident radiation} 
\begin{itemize}
\item [{\bf Inputs}:] Underlying domain $\domain$, coefficients and data for the radiative transfer equation \eqref{eq:genRTE}, measured incident radiation $G$, quadrature points and weights for underlying quadrature rules, non-convex gradient based optimization algorithms.
\item [{\bf Goal}:] Find PINN $(k^{\ast},u^{\ast})= \left(k_{\theta_k^{\ast}}, u_{\theta_u^{\ast}}\right)$ for approximating the inverse problem for radiative transfer
\item [{\bf Step $1$}:] Choose the training sets as described in section \ref{sec:22} and in this subsection.
\item [{\bf Step $2$}:] For an initial value of the weight vector $\overline{\theta} \in \Theta$, evaluate the neural networks $u_{\overline{\theta}_u}$, $k_{\overline{\theta}_k}$ \eqref{eq:ann1}, the PDE residual \eqref{eq:res1i}, the data residual \eqref{eq:resd}, the boundary residuals \eqref{eq:resb}, the loss function \eqref{eq:lf3} and its gradients to initialize the underlying optimization
algorithm.
\item [{\bf Step $3$}:] Run the optimization algorithm till an approximate local minimum $\theta^{\ast}$ of \eqref{eq:lf3} is reached. The map $u^{\ast} = u_{\theta^{\ast}_u}$ is the desired PINN for approximating the solution $u$ of the radiative transfer equation and the map $k^{\ast} = k_{\theta_k^{\ast}}$ is the corresponding absorption coefficient. 
\end{itemize}
\end{algorithm}
\begin{table}[htbp] 
    \centering
    \renewcommand{\arraystretch}{1.1} 
    
    \footnotesize{
        \begin{tabular}{c  c c c c c  c c c c c  c c} 
            \toprule
             $N_{int}$  & $N_{sb}$  & $N_{d}$  &\bfseries $K-1$ & \bfseries $\tilde{d}$  &\bfseries $\lambda$&  $\er_T$ & $||u - u^{\ast}||_{L^2}$   & $||k - k^{\ast}||_{L^2}$ & $||G - G^{\ast}||_{L^2}$   & Training Time \\ 
            \midrule
            \midrule
            16384 & 120   & 4096 &8&20& 1.0   & 0.00094& 0.65 $\%$ &2.8 $\%$ &0.073$\%$ & 1~hr~44~min \\
            \bottomrule
        \end{tabular}
        \caption{Results for the inverse problem for radiative transfer. }
        \label{tab:r1}
    }
\end{table}
\subsection{A Numerical Experiment.}
\label{sec:42}
The monochromatic stationary version of the radiative transfer equation \eqref{eq:genRTE} in three space dimensions, is used in this numerical experiment. The spatial domain is the unit cube $D = [0,1]^3$, with scattering coefficient $\sigma=0.5$, scattering kernel $\Phi \equiv 1$. The source term $f$ and boundary term $u_b$ are generated using the following synthetic absorption coefficient and exact solution,
\begin{equation}
\label{eq:inv}
  k(x) = \prod_{i=1}^{3}x_i^2,\quad   u(x,\omega) = \frac{3}{16\pi}(1 + (\omega \cdot \omega')^2)\prod_{i=1}^{3}x_i(x_i-1), \quad n'=\Big(\frac{1}{\sqrt{3}}, \frac{1}{\sqrt{3}}, \frac{1}{\sqrt{3}}\Big)^T
\end{equation}
The measured incident radiation $\bar{G}$ in \eqref{eq:resd} is calculated from the radiative intensity $u$ above by using the formula \eqref{eq:inc_rad}. 

For this numerical experiment, we also impose boundary conditions on the neural network approximating the absorption coefficient $k_{\theta_k}$ to approximately match the values of $k$, defined in \eqref{eq:inv} on the boundary of $D$, leading to an additional term in \eqref{eq:lf3}. Finally, in order to ensure uniqueness of the absorption coefficient, we include in the loss function, the so-called Tikhonov regularization:
\begin{equation}
    J_T(\theta) = \lambda_k ||\nabla k_\theta||^2_2, \quad \lambda_k=0.001.
\end{equation}
We use Sobol points for the interior training points and uniformly distributed random points are used as data training points, with $N_{int}=16384$, $N_{d} = 4096$. The resulting best performing hyperparameter configuration after ensemble training is presented in Table \ref{tab:r1}. 

In figure \ref{fig:r1} we plot the incident radiation $G$ and the absorption coefficient $k$, along the diagonal of the unit cube, computed with the PINNs algorithm \ref{alg:PINNi}. As observed from this figure, the incident radiation is almost identical to the measured data $\bar{G}$. This is further verified from table \ref{tab:r1}, from which we observe a very low $L^2$-error for the incident radiation. On the other hand, the absorption coefficient agrees reasonably well with the ground truth in \eqref{eq:inv}, with an error of less than $3\%$. Also the radiative intensity is approximated to very high accuracy, with a generalization error below $1\%$. This is even more impressive if one consider that the problem is solved with a computational time of approximately 100 minutes. 
 
\begin{figure}[ht!]

    \begin{subfigure}{.49\textwidth}
        \centering
        \includegraphics[width=1\linewidth]{{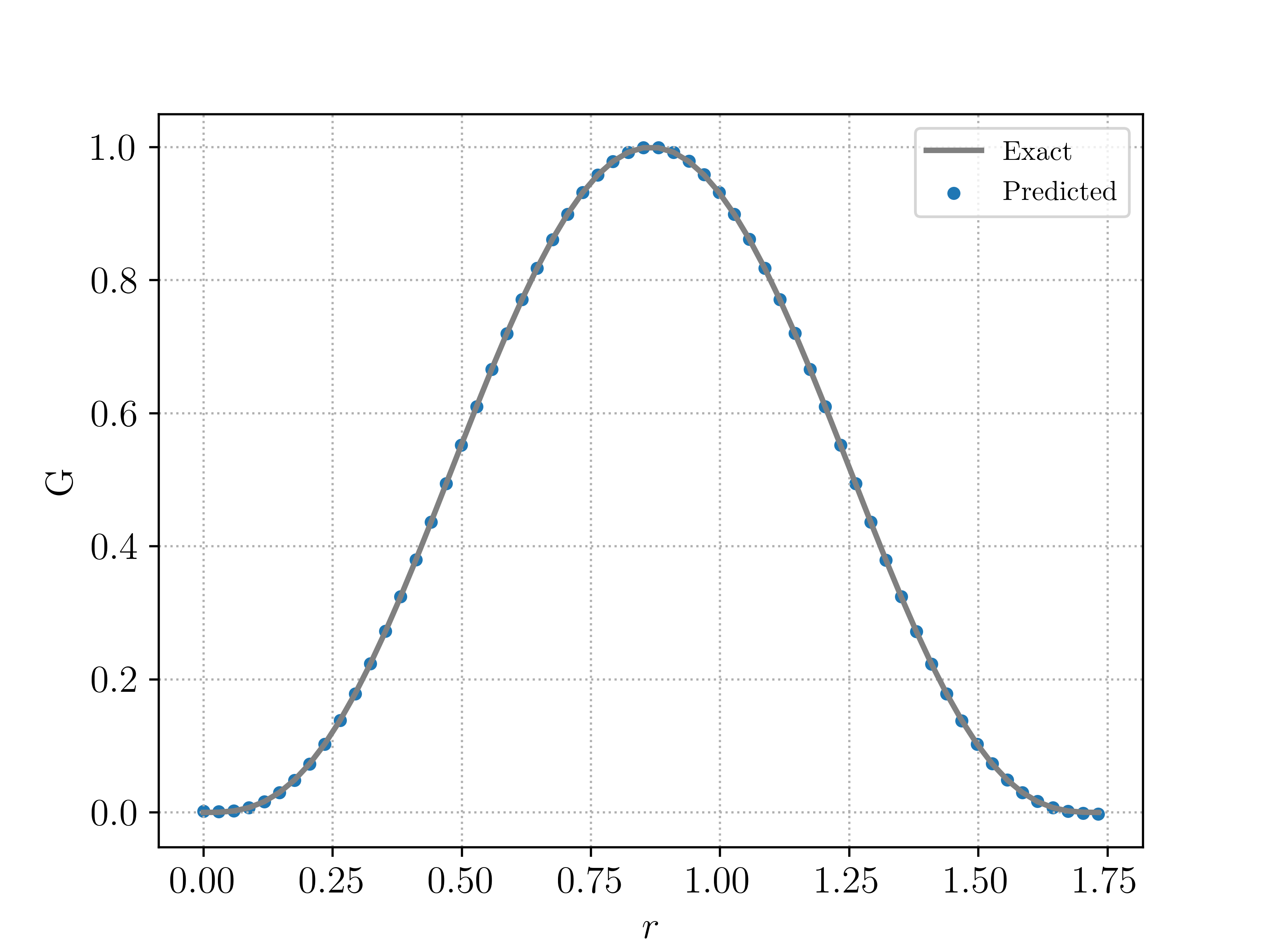}}
        \caption{Incident radiation $G$}
    \end{subfigure}
    \begin{subfigure}{.49\textwidth}
        \centering
        \includegraphics[width=1\linewidth]{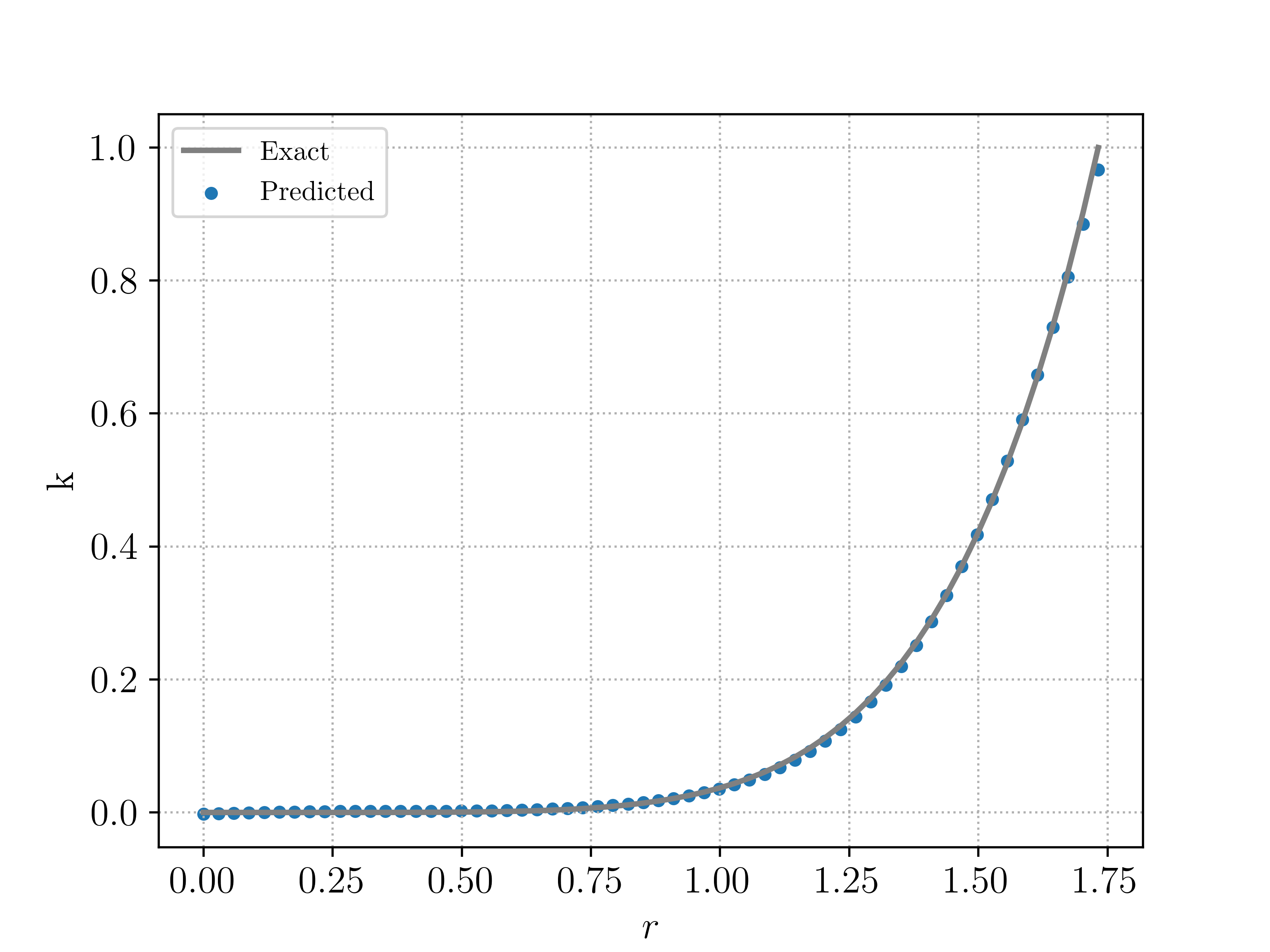}
        \caption{Absorption Coefficient $k$}
    \end{subfigure}
    \caption{Results for the PINNs algorithm \ref{alg:PINNi} for the inverse problem for radiative transfer. The PINNs approximation to the incident radiation and absorption coefficient are plotted along the diagonal of the unit cube and compared with the measured data $\bar{G}$ and ground truth absorption coefficient $k$ given by \eqref{eq:inv}.}
\label{fig:r1}
\end{figure} 
\section{Discussion}
\label{sec:5}
Accurate numerical approximation of the radiative transfer equations \eqref{eq:genRTE} is considered very challenging as the underlying problem is high-dimensional, with $7$-dimensions in the most general case. Moreover, the presence of different physical effects such as emission, absorption and scattering as well as varying optical parameters in the surrounding medium further complicates design of efficient numerical algorithms. As discussed in the introduction, existing numerical methods can suffer from the so-called \emph{curse of dimensionality} and require very large amount of computational resources to achieve desired accuracy. 

Hence, there is a need for designing algorithms for simulating radiative transfer, that are easy to implement and fast (in terms of computational run time) while still being accurate. We proposed such an algorithm in this article. Our algorithm \ref{alg:PINN} is based on \emph{physics informed neural networks} (PINNs) i.e. deep neural networks for approximating the radiative intensity in \eqref{eq:genRTE}. The deep neural network is trained by using a gradient descent method to minimize a loss function \eqref{eq:lf2}, that consists of the \emph{PDE residual}, resulting from the neural network being plugged into the radiative transfer equation \eqref{eq:genRTE}. Mismatches with respect to the initial and boundary conditions also contribute to the loss function. The residuals are collocated at \emph{training points}, which correspond to quadrature points with respect to an underlying quadrature rule. We chose Sobol low-discrepancy sequences as training points in order to alleviate the curse of dimensionality. 

The resulting algorithm is extremely simple to code within standard machine learning frameworks such as TensorFlow and Pytorch. We presented a suite of numerical experiments ranging from the simplest monochromatic stationary radiative transfer in one space dimension \eqref{eq:1drad} to the most general time-dependent polychromatic radiative transfer in three space dimensions. PINNs performed very well on all the numerical experiments, leading to low errors with  small run (training) times. In particular, the results were qualitatively and quantitatively comparable to published results, but possibly at a fraction of the cost. The experimental results were supplemented with rigorous error estimates that bounded the generalization error \eqref{eq:gerr} in terms of computable training errors \eqref{eq:etrain} and number of quadrature points, independent of the underlying dimension, see bounds \eqref{eq:egbd} for details. The predictions of the error estimates were validated by the experiments. 

Hence, we claim that the PINNs algorithm \ref{alg:PINN} is a general purpose, simple to implement, fast and accurate simulator for radiative transfer. Moreover, we also presented a (very) slightly modified version of the PINN algorithm to efficiently simulate a class of \emph{inverse problems} for radiative transfer. In this problem, the objective is to compute the unknown absorption coefficient from measurements of the incident radiation. To this end, we proposed the deep neural network based algorithm \ref{alg:PINNi}, which added a data fidelity term \eqref{eq:resd} to the underlying loss function \eqref{eq:lf3}. This algorithm was also found to be both fast and accurate in numerical experiments. Thus, we provide \emph{novel machine learning algorithms} which are fast, easy to implement and accurate for efficiently simulating different aspects of radiative transfer.

This article should be considered as a first step in adapting machine learning algorithms for simulating radiative transfer. The proposed algorithm lends itself readily to be extended in the following directions,
\begin{itemize}
    \item The quantities of interest in many applications of radiative transfer, particularly in astrophysics, are angular moments such as the incident radiation $G$ and heat flux $F$, defined in \eqref{eq:inc_rad} as these quantities define the contribution of radiation to the total energy. Thus, in radiation hydrodynamics, one often resorts to using moment models. However, these moment models require closure relations, which might either be expensive to compute and/or inaccurate (see section \ref{sec:35} for a simple diffusion based closure). Given the computational speed of PINNs, one can readily employ algorithm \ref{alg:PINN} as the radiation module in a hydrodynamics code, with angular moments and the resulting energy being computed from the neural network approximation of the radiative intensity. Another option would be to leverage the ability of PINNs to directly perform hydrodynamic simulations (see \cite{KAR4,MM1}) and have a full PINN simulation of radiation hydrodynamics. Both approaches should be developed and tested. 
    \item We covered the inverse problem for radiative transfer briefly in section \ref{sec:4}. Algorithm \ref{alg:PINNi} can be readily extended to many other inverse problems involving radiative transfer, for instance finding emission and scattering coefficients from measurements of incident radiation and heat fluxes at a few points. A careful exposition and analysis of the resulting algorithms and their application in practical contexts will be the basis of future papers. 
\end{itemize}
\section*{Acknowledgements}
The research of SM and RM was partially supported by the European Research Council (ERC) under ERCCoG 770880 COMANFLO. SM thanks Dr. Roger K\"appeli of SAM, ETH Z\"urich for his inputs on radiation and astrophysics.

\appendix
\section{Estimates on the generalization error for the radiative transfer equation \eqref{eq:genRTE}}
\label{sec:a1}
In order to derive an error estimate for the PINNs algorithm, we need to make some assumptions on the scattering kernel $\Phi$ in \eqref{eq:genRTE}. We follow standard practice and assume that it is symmetric $\Phi(\omega,\omega^{\prime},\nu,\nu^{\prime}) = \Phi(\omega^{\prime},\omega,\nu^{\prime},\nu)$. Moreover, the following function,
\begin{equation}
    \label{eq:psi}
    \Psi(\omega,\nu) = \int\limits_{S \times \Lambda} \Phi(\omega,\omega^{\prime},\nu,\nu^{\prime}) d\omega^{\prime} d\nu^{\prime},
\end{equation}
is essentially bounded i.e. $\Psi \in L^{\infty}(S \times \Lambda)$.
We have the following estimate on the generalization error \eqref{eq:gerr},
\begin{lemma}
\label{lem:1}
Let $u \in L^2(\dom)$ be the unique weak solution of the radiative transfer equation \eqref{eq:genRTE}, with absorption coefficient $0 \leq k \in L^{\infty}(D\times \Lambda)$, scattering coefficient $0 \leq \sigma \in L^{\infty}(D\times \Lambda)$ and a symmetric scattering kernel $\Phi \in C^\ell(S \times \Lambda \times S \times \Lambda)$, for some $\ell \geq 1$, such that the function $\Psi$ (defined in \eqref{eq:psi}) is in $L^{\infty}(S \times \Lambda)$. Let $u^{\ast} = u_{\theta^{\ast}} \in C^{\ell}(\dom)$ be the output of the PINNs algorithm \ref{alg:PINN} for approximating the radiative transfer equation \eqref{eq:genRTE}, such that 
\begin{equation}
    \max\{ V_{HK}(u^{\ast}), V_{HK}\left(\res_{int,\theta^{\ast}}\right)\} < +\infty, 
\end{equation}
with $V_{HK}$ being the so-called Hardy-Krause variation (see \cite{CAF1,MR1} for the precise definition). We also assume that the initial data $u_0$ and boundary data $u_b$ are of bounded Hardy-Krause variation. Then, under the assumption that Sobol points are used as the training points $\train_{int},\train_{sb},\train_{tb}$ in algorithm \ref{alg:PINN} and Guass-quadrature rule of order $s = s(\ell)$ is used in approximating the scattering kernel in the residual \eqref{eq:res1}, we have the following estimate on the generalization error,
\begin{equation}
    \label{eq:egbd}
    \begin{aligned}
    (\er_G)^2 &\leq C\left((\er_T^{tb})^2 + c(\er_T^{sb})^2 + c(\er_T^{int})^2\right) \\
    &+ CC^{\ast}\left(\frac{(\log(N_{tb}))^{2d}}{N_{tb}} + c \frac{(\log(N_{sb}))^{2d}}{N_{sb}}+ c \frac{(\log(N_{int}))^{2d+1}}{N_{int}} + c N^{-2s}_{S}\right) 
    \end{aligned}
\end{equation}
with constants defined as,
\begin{equation}
    \label{eq:c}
    \begin{aligned}
    C &= T+c\hat{C}T^2e^{c\hat{C}T}, \quad \hat{C} = 2 +  \frac{2(\|\sigma\|_{L^{\infty}}+\|\Psi\|_{L^{\infty}})}{s_d} \\
    C^{\ast}&=\max\left\{V_{HK}\left((\res_{tb}^{\ast})^2\right), V_{HK}\left((\res_{sb}^{\ast})^2\right), V_{HK}\left((\res_{int}^{\ast})^2\right),\overline{C}\right\} \\
    \overline{C} &= \overline{C}\left(|\dom|,\|\Phi\|_{C^{\ell}},\|u^{\ast}\|_{C^{\ell}}\right)
     \end{aligned}
\end{equation}
\end{lemma}
\begin{proof}
We drop the $\theta^{\ast}$ dependence in the residuals \eqref{eq:res1}, \eqref{eq:resb}, for notational convenience and denote the residuals as $\res^{\ast}_{int},\res^{\ast}_{sb},\res^{\ast}_{tb}$. Define, 
\begin{equation}
    \label{eq:E}
    E(u^{\ast},\Phi):= \sum\limits_{i=1}^{N_S} w_i^S \Phi(\omega, \omega^S_i, \nu, \nu^S_i) u^{\ast}(t,x,\omega^{S}_i, \nu^{S}_i) - \int\limits_{\Lambda}\int\limits_{S} \Phi(\omega, \omega^{\prime}, \nu, \nu^{\prime})u^{\ast}(t,x,\omega^{\prime}, \nu^{\prime}) d\omega^{\prime}d\nu^{\prime}.
\end{equation}
It is straightforward to derive from the radiative transfer equation \eqref{eq:genRTE} and the definition of residuals \eqref{eq:res1}, \eqref{eq:resb}, that the error $\hat{u} = u^{\ast} - u$, satisfies the following integro-differential equation,
\begin{equation}
    \label{eq:uhat}
    \begin{aligned}
    \frac{1}{c}\hat{u}_t + \omega\cdot\nabla_x \hat{u}&=   -(k+ \sigma) \hat{u} +   \frac{\sigma}{s_d}\int\limits_{\Lambda}\int\limits_{S} \Phi(\omega, \omega^{\prime}, \nu, \nu^{\prime})\hat{u}(t,x,\omega^{\prime}, \nu^{\prime}) d\omega^{\prime}d\nu^{\prime} \\
    &+ \res_{int}^{\ast} + E(u^{\ast},\Phi). \\
    \hat{u}(0,x,\omega,\nu) &= \res^{\ast}_{tb}, \quad (x,\omega,\nu) \in D \times S \times \Lambda, \\
    \hat{u}(t,x,\omega,\nu) &= \res^{\ast}_{sb}, \quad (t,x,\omega,\nu) \in \Gamma_- \times \Lambda.
\end{aligned}
\end{equation}
Multiplying $\hat{u}$ on both sides of the first equation in \eqref{eq:uhat}, we obtain, 
\begin{equation}
    \label{eq:pf1}
    \begin{aligned}
    \frac{1}{2c}\frac{d(\hat{u}^2)}{dt} + \omega\cdot\nabla_x (\frac{\hat{u}^2}{2})&=   -(k+ \sigma) \hat{u}^2 +   \frac{\sigma}{s_d}\int\limits_{\Lambda}\int\limits_{S} \Phi(\omega, \omega^{\prime}, \nu, \nu^{\prime})\hat{u}(t,x,\omega^{\prime}, \nu^{\prime})\hat{u}(t,x,\omega,\nu)  d\omega^{\prime}d\nu^{\prime} \\
    &+ \res_{int}^{\ast}\hat{u} + E(u^{\ast},\Phi)\hat{u}
\end{aligned}
\end{equation}
Integrating the above over $D \times S \times \nu$, integrating by parts and using the Cauchy's inequality and the fact that $k,\sigma \geq 0$, we obtain for any $t \in (0,T]$,
\begin{equation}
    \label{eq:pf2}
    \begin{aligned}
    \frac{1}{2c}\frac{d}{dt}\int_{D\times S \times \Lambda} \hat{u}^2(t,x,\omega,\nu) dx d\omega d\nu &\leq  \int_{D\times S \times \Lambda} \hat{u}^2(t,x,\omega,\nu) dx d\omega d\nu - \int\limits_{(\partial D \times S \times \Lambda)_-} (\omega \cdot n(x)) \frac{\hat{u}^2(t,x,\omega,\nu)}{2} ds(x)d\omega d\nu \\
    &+\int\limits_{D\times S \times \Lambda}\frac{\sigma}{s_d}\int\limits_{\Lambda}\int\limits_{S} \Phi(\omega, \omega^{\prime}, \nu, \nu^{\prime})\hat{u}(t,x,\omega^{\prime}, \nu^{\prime})\hat{u}(t,x,\omega,\nu) d\omega^{\prime}d\nu^{\prime} d\nu d\omega dx, \\
    &+\int\limits_{D\times S \times \Lambda}\frac{(\res^{\ast}_{int}(t,x,\omega,\nu))^2}{2} d\nu d\omega dx + \int\limits_{D\times S \times \Lambda}\frac{(E(u^{\ast},\Phi)(t,x,\omega,\nu))^2}{2} d\nu d\omega dx
    \end{aligned}
\end{equation}
Here $ds(x)$ denotes the surface measure on $\partial D$ and we define
$$
(\partial D \times S \times \Lambda)_- := \{(x,\omega,\nu) \in \partial D \times S \times \Lambda: \omega\cdot n(x) \leq 0\},
$$
with $n(x)$ being the unit outward normal at $x \in \partial D$.

We fix any $\bar{T} \in (0,T]$ and integrate \eqref{eq:pf2} over $(0,\bar{T})$ and estimate the result to obtain,
\begin{equation}
\label{eq:pf4}
\begin{aligned}
\int\limits_{D\times S \times \Lambda} \hat{u}^2(\bar{T},x,\omega,\nu) dx d\omega d\nu
&\leq \int\limits_{D\times S \times \Lambda} \hat{u}^2(0,x,\omega,\nu) dx d\omega d\nu + 2c\int\limits_0^{\bar{T}} \int_{D\times S \times \Lambda} \hat{u}^2(t,x,\omega,\nu) dt dx d\omega d\nu \\
&+ c\int\limits_{\Gamma_-} |\omega \cdot n|\hat{u}^2(t,x,\omega,\nu) dt ds(x) d\omega d\nu 
+ I + c\int\limits_{\dom} (\res^{\ast}_{int})^2
dz + c\int\limits_{\dom} (E(u^{\ast},\Phi))^2
dz.
\end{aligned}
\end{equation}
Here, the term $I$ in \eqref{eq:pf4}, is defined and estimated by successive applications of Cauchy-Schwatrz inequality as,
\begin{align*}
    I &= 2c\int\limits_0^{\bar{T}} \int\limits_{D\times S \times \Lambda}\frac{\sigma}{s_d}\int\limits_{\Lambda}\int\limits_{S} \Phi(\omega, \omega^{\prime}, \nu, \nu^{\prime})\hat{u}(t,x,\omega^{\prime}, \nu^{\prime})\hat{u}(t,x,\omega,\nu) d\omega^{\prime}d\nu^{\prime} d\nu d\omega dx dt, \\
    &\leq \frac{2c(\|\sigma\|_{L^{\infty}}+\|\Psi\|_{L^{\infty}})}{s_d}\int\limits_0^{\bar{T}} \int_{D\times S \times \Lambda} \hat{u}^2(t,x,\omega,\nu) dt dx d\omega d\nu.
\end{align*}
By identifying constant $\hat{C}$ from \eqref{eq:c},
we obtain from \eqref{eq:pf4} and \eqref{eq:uhat} that,
\begin{equation}
\label{eq:pf5}
\begin{aligned}
\int\limits_{D\times S \times \Lambda} \hat{u}^2(\bar{T},x,\omega,\nu) dx d\omega d\nu
&\leq \int\limits_{D\times S \times \Lambda} (\res^{\ast}_{tb})^2 dx d\omega d\nu + c\int\limits_{\Gamma_-}  (\res^{\ast}_{sb})^2 dt ds(x) d\omega d\nu \\
&+ c\int\limits_{\dom} (\res^{\ast}_{int})^2
dz + c\int\limits_{\dom} (E(u^{\ast},\Phi))^2
dz \\
&+ c\hat{C}\int\limits_0^{\bar{T}} \int\limits_{D\times S \times \Lambda} \hat{u}^2(t,x,\omega,\nu) dt dx d\omega d\nu.
\end{aligned}
\end{equation}
Applying the integral form of Gr\"onwall's inequality to \eqref{eq:pf5}, we obtain for any $0 < \bar{T} \leq T$,
\begin{equation}
    \label{eq:pf6}
    \begin{aligned}
    \int\limits_{D\times S \times \Lambda} \hat{u}^2(\bar{T},x,\omega,\nu) dx d\omega d\nu
&\leq \left(1+c\hat{C}\bar{T}e^{c\hat{C}\bar{T}}\right)\left(\int\limits_{D\times S \times \Lambda} (\res^{\ast}_{tb})^2 dx d\omega d\nu + c\int\limits_{\Gamma_-}  (\res^{\ast}_{sb})^2 dt ds(x) d\omega d\nu \right) \\
&+ \left(1+c\hat{C}\bar{T}e^{c\hat{C}\bar{T}}\right)\left(c\int\limits_{\dom} (\res^{\ast}_{int})^2
dz + c\int\limits_{\dom} (E(u^{\ast},\Phi))^2
dz \right)
    \end{aligned}
\end{equation}
Integrating \eqref{eq:pf6} over $(0,T)$ yields,
\begin{equation}
    \label{eq:pf7}
    \begin{aligned}
    (\er_G)^2 := \int\limits_{\dom} \hat{u}^2(t,x,\omega,\nu) dz 
&\leq \left(T+c\hat{C}T^2e^{c\hat{C}T}\right)\left(\int\limits_{D\times S \times \Lambda} (\res^{\ast}_{tb})^2 dx d\omega d\nu + c\int\limits_{\Gamma_-}  (\res^{\ast}_{sb})^2 dt ds(x) d\omega d\nu \right) \\
&+ \left(T+c\hat{C}T^2e^{c\hat{C}T}\right)\left(c\int\limits_{\dom} (\res^{\ast}_{int})^2
dz + c\int\limits_{\dom} (E(u^{\ast},\Phi))^2
dz \right)
    \end{aligned}
\end{equation}
As the training points in $\train_{tb}$ are the Sobol quadrature points, we realize that the training error $(\er_T^{tb})^2$ \eqref{eq:etrain} is the quasi-Monte Carlo quadrature for the first integral in \eqref{eq:pf7}. Hence by the well-known Koksma-Hlawka inequality \cite{CAF1}, we obtain the following estimate,
\begin{equation}
    \label{eq:pf8}
    \int\limits_{D\times S \times \Lambda} (\res^{\ast}_{tb})^2 dx d\omega d\nu \leq (\er_T^{tb})^2 + V_{HK}\left((\res^{\ast}_{tb})^2\right)\frac{(\log(N_{tb}))^{2d}}{N_{tb}}.
\end{equation}
By a similar argument, we can estimate,
\begin{equation}
    \label{eq:pf9}
    \begin{aligned}
    \int\limits_{\Gamma_-}  (\res^{\ast}_{sb})^2 dt ds(x) d\omega d\nu &\leq (\er_T^{sb})^2 + V_{HK}\left((\res^{\ast}_{sb})^2\right)\frac{(\log(N_{sb}))^{2d}}{N_{sb}}, \\
    \int\limits_{\dom} (\res^{\ast}_{int})^2
dz &\leq (\er_T^{int})^2 + V_{HK}\left((\res^{\ast}_{int})^2\right)\frac{(\log(N_{int}))^{2d+1}}{N_{int}}, 
    \end{aligned}
\end{equation}
As $\omega_i^S,\nu_i^S$, for $1 \leq i \leq N_S$ are Gauss-quadrature points, we follow \cite{SBbook} and readily estimate $E$ defined in \eqref{eq:E} by the error for an $s$-th order accurate Gauss quadrature rule with $s = s(\ell)$ as, 
\begin{equation}
   \label{eq:pf10}
   \int\limits_{\dom} (E(u^{\ast},\Phi))^2
dz \leq \overline{C} N_{S}^{-2s},
\end{equation}
with constant $\overline{C}$ defined in \eqref{eq:c} 
By plugging in the estimates \eqref{eq:pf8}, \eqref{eq:pf9}, \eqref{eq:pf10} in \eqref{eq:pf7} and identifying constants, we derive the desired estimate \eqref{eq:egbd} on the generalization error \eqref{eq:gerr}. 
\end{proof}
\section{Estimates on the generalization error in the steady case}
\label{sec:a2}
The steady-state (time-independent) version of the radiative transfer equation \eqref{eq:genRTE} is obtained by letting the speed of light $c \rightarrow \infty$ and resulting in,
\begin{equation}
\label{eq:sRTE}
\begin{aligned}
  (k+\sigma) u = - \omega\cdot\nabla_x u +  \frac{\sigma}{s_d}\int\limits_{\Lambda}\int\limits_{S} \Phi(\omega, \omega^{\prime}, \nu, \nu^{\prime})u(x,\omega^{\prime}, \nu^{\prime}) d\omega^{\prime}d\nu^{\prime} + f,
\end{aligned}
\end{equation}
with all the coefficients and sources as defined before. We also impose the \emph{inflow} boundary condition,
\begin{equation}
    \label{eq:sbc}
 u(x,\omega,\nu) = u_b(x,\omega,\nu), \quad (t, x,\omega, \nu)\in\Gamma^s_,
\end{equation}
with inflow boundary defined by,
\begin{equation}
    \Gamma^s_-= \{(x,\omega,\nu)\in \partial D\times S \times \Lambda: \omega \cdot n(x) <0\}
\end{equation}
with $n(x)$ denoting the unit outward normal at any point $x \in \partial D$.

The PINNs algorithm \ref{alg:PINN} can be readily adpated to this case by simply (formally) neglecting the temporal dependence in the residuals \eqref{eq:res1}, \eqref{eq:resb} and loss functions and the underlying definitions of neural networks. We omit detailing this procedure here. Our objective is to bound the resulting generalization error,
\begin{equation}
    \label{eq:gerrs}
\er^s_{G} = \er^s_{G}(\theta^{\ast}):=\left(\int\limits_{D\times S \times \Lambda}|u(x,\omega,\nu) - u^{\ast}(x,\omega,\nu)|^2 dz\right)^{\frac{1}{2}},
\end{equation}
with $dz = dx d\omega d\nu$ denoting the underlying volume measure. As in lemma \ref{lem:1}, we will bound the generalization error in terms of the training errors,
\begin{equation}
    \label{eq:estrain}
    \er_T^{sb}:= 
   \left(\sum\limits_{j=1}^{N_{sb}} w_j^{sb} |\res_{sb,\theta^{\ast}}(z_j^{sb})|^{2}\right)^{\frac{1}{2}}, \quad 
    \er_T^{int}:= \left(\sum\limits_{j=1}^{N_{int}} w^{int}_j |\res_{int,\theta^{\ast}}(z_j^{int})|^{2}\right)^{\frac{1}{2}} 
\end{equation}
Here, $z_j^{int}$ and $z_j^{sb}$ are the interior and spatial boundary training points. 

We have the following estimate on the generalization error,
\begin{lemma}
\label{lem:2}
Let $u \in L^2(D \times S \times \Lambda)$ be the unique weak solution of the radiative transfer equation \eqref{eq:sRTE}, with absorption coefficient $0 < k_{min} \leq k(x,\nu) \leq k_{max} < \infty $, scattering coefficient $0 < \sigma_{min}  \leq \sigma(x,\nu) \leq \sigma_{max} < \infty$, for almost every $x \in D, \nu \in \Lambda$ and a symmetric scattering kernel $\Phi \in C^\ell(S \times \Lambda \times S \times \Lambda)$, for some $\ell \geq 1$, such that the function $\Psi$ (defined in \eqref{eq:psi}) is in $L^{\infty}(S \times \Lambda)$. We further assume that the absorption and scattering coefficients are related in the following manner, there exists a $\kappa > 0$, such that 
\begin{equation}
    \label{eq:assm}
    k_{min} + \sigma_{min} - \frac{\sigma_{max}+\|\Psi\|_{L^{\infty}}}{s_d} \geq \kappa
    \end{equation}
Let $u^{\ast} = u_{\theta^{\ast}} \in C^{\ell}(D \times S \times \Lambda)$ be the output of the PINNs algorithm \ref{alg:PINN} for approximating the stationary radiative transfer equation \eqref{eq:sRTE}, such that 
\begin{equation}
    \max\{ V_{HK}(u^{\ast}), V_{HK}\left(\res_{int,\theta^{\ast}}\right)\} < +\infty, 
\end{equation}
with $V_{HK}$ being the Hardy-Krause variation.  We also assume that the boundary data $u_b$ is of bounded Hardy-Krause variation. Then, under the assumption that Sobol points are used as the training points $\train_{int},\train_{sb}$ in algorithm \ref{alg:PINN} and Guass-quadrature rule of order $s = s(\ell)$ is used in approximating the scattering kernel in the residual \eqref{eq:res1}, we have the following estimate on the generalization error,
\begin{equation}
    \label{eq:segbd}
    \begin{aligned}
    (\er^s_G)^2 &\leq C\left((\er_T^{sb})^2 + (\er_T^{int})^2 + \frac{(\log(N_{sb}))^{2d-1}}{N_{sb}}+  \frac{(\log(N_{int}))^{2d}}{N_{int}} + N^{-2s}_{S}\right) 
    \end{aligned}
\end{equation}
with constants defined as,
\begin{equation}
    \label{eq:sc}
    \begin{aligned}
      C = \max\left\{\frac{2}{\kappa}, \frac{2}{\kappa}V_{HK}\left((\res^{\ast}_{sb})^2\right),\frac{2C_{\epsilon}}{\kappa}\left((\res^{\ast}_{int})^2\right),\frac{2C_{\epsilon}}{\kappa}\overline{C}N_S^{-2s}\right\},
      \end{aligned}
\end{equation}
where $\overline{C}$ is defined in \eqref{eq:c}. Here, $C_{\epsilon}$ is a constant that depends on $\kappa$ and is defined in \eqref{eq:ec}. 
\end{lemma}
\begin{proof}
We drop the $\theta^{\ast}$ dependence in the residuals \eqref{eq:res1}, \eqref{eq:resb}, for notational convenience and denote the residuals as $\res^{\ast}_{int},\res^{\ast}_{sb}$. Define, 
\begin{equation}
    \label{eq:Es}
    E_s(u^{\ast},\Phi):= \sum\limits_{i=1}^{N_S} w_i^S \Phi(\omega, \omega^S_i, \nu, \nu^S_i) u^{\ast}(x,\omega^{S}_i, \nu^{S}_i) - \int\limits_{\Lambda}\int\limits_{S} \Phi(\omega, \omega^{\prime}, \nu, \nu^{\prime})u^{\ast}(x,\omega^{\prime}, \nu^{\prime}) d\omega^{\prime}d\nu^{\prime}.
\end{equation}
It is straightforward to derive from the radiative transfer equation \eqref{eq:sRTE} and the definition of residuals \eqref{eq:res1}, \eqref{eq:resb}, that the error $\hat{u} = u^{\ast} - u$, satisfies the following integro-differential equation,
\begin{equation}
    \label{eq:ushat}
    \begin{aligned}
    (k+ \sigma) \hat{u} &= - \omega\cdot\nabla_x \hat{u} +  \frac{\sigma}{s_d}\int\limits_{\Lambda}\int\limits_{S} \Phi(\omega, \omega^{\prime}, \nu, \nu^{\prime})\hat{u}(t,x,\omega^{\prime}, \nu^{\prime}) d\omega^{\prime}d\nu^{\prime} + \res_{int}^{\ast} + E_s(u^{\ast},\Phi), \\
    &\hat{u}(t,x,\omega,\nu) = \res^{\ast}_{sb}, \quad (x,\omega,\nu) \in \Gamma^s_-
\end{aligned}
\end{equation}
Multiplying $\hat{u}$ on both sides of the first equation in \eqref{eq:ushat}, we obtain, 
\begin{equation}
    \label{eq:spf1}
    \begin{aligned}
 (k+ \sigma) \hat{u}^2  &= -\omega\cdot\nabla_x (\frac{\hat{u}^2}{2}) +   \frac{\sigma}{s_d}\int\limits_{\Lambda}\int\limits_{S} \Phi(\omega, \omega^{\prime}, \nu, \nu^{\prime})\hat{u}(t,x,\omega^{\prime}, \nu^{\prime})\hat{u}(t,x,\omega,\nu)  d\omega^{\prime}d\nu^{\prime} \\
    &+ \res_{int}^{\ast}\hat{u} + E_s(u^{\ast},\Phi)\hat{u}
\end{aligned}
\end{equation}
Integrating the above over $D \times S \times \nu$, integrating by parts, using the assumed lower and upper bounds on $k,\sigma$, we obtain,
\begin{equation}
\label{eq:spf2}
    \begin{aligned}
    (k_{min}+\sigma_{min})\int\limits_{D \times S \times \nu} \hat{u}^2 dz &\leq \int\limits_{\Gamma^s_-}  (\res^{\ast}_{sb})^2 ds(x) d\omega d\nu + I + \int\limits_{D \times S \times \nu} (\res_{int}^{\ast}\hat{u} + E_s(u^{\ast},\Phi)\hat{u}) dz,
    \end{aligned}
\end{equation}
with term $I$ defined and estimated by,
\begin{align*}
    I &= \int\limits_{D\times S \times \Lambda}\frac{\sigma}{s_d}\int\limits_{\Lambda}\int\limits_{S} \Phi(\omega, \omega^{\prime}, \nu, \nu^{\prime})\hat{u}(x,\omega^{\prime}, \nu^{\prime})\hat{u}(x,\omega,\nu) d\omega^{\prime}d\nu^{\prime} d\nu d\omega dx , \\
    &\leq \frac{\sigma_{max}+\|\Psi\|_{L^{\infty}}}{s_d}\int_{D\times S \times \Lambda} \hat{u}^2(x,\omega,\nu) dz.
\end{align*}
From the assumption \eqref{eq:assm}, there exists an $\epsilon > 0$ such that $k_{min} + \sigma_{min} - \frac{\sigma_{max}+\|\Psi\|_{L^{\infty}}}{s_d} -2 \epsilon > \frac{\kappa}{2}$, we use the $\epsilon$-version of Cauchy's inequality,
\begin{equation}
    \label{eq:ec}
    ab \leq \epsilon a^2 + C_{\epsilon} b^2, 
\end{equation}
to further estimate \eqref{eq:spf2} as, 
\begin{equation}
    \label{eq:spf3}
    \int\limits_{D\times S \times \Lambda} \hat{u}^2 dz \leq \frac{2}{\kappa} \int\limits_{\Gamma^s_-}  (\res^{\ast}_{sb})^2 ds(x) d\omega d\nu + \frac{2C_{\epsilon}}{\kappa}\left(\int\limits_{D \times S \times \nu} (\res_{int}^{\ast})^2 + (E_s(u^{\ast},\Phi))^2 dz \right)
\end{equation}
By using the estimates \eqref{eq:pf9} and \eqref{eq:pf10} and identifying constants, we obtain the desired bound \eqref{eq:segbd} on the generalization error \eqref{eq:gerrs}.

\end{proof}
As for the time-dependent case, the bound \eqref{eq:segbd} should be considered in the sense of \emph{if the PINN is trained well, it generalizes well}. Moreover, the bound, and consequently, the PINN does not suffer from a curse of dimensionality by the same argument as in the time-dependent case. Infact, the logarithmic corrections to the linear decay of the rhs in \eqref{eq:segbd} can be ignored at an even smaller number of training points.

The assumption \eqref{eq:assm} plays a key role in the derivation of the bound \eqref{eq:segbd}. A careful inspection of this assumption reveals that the scattering coefficient is not allowed to vary over a large range, unless there is enough absorption in the medium. However, there is no restriction on the range of scales over which the absorption coefficient can vary.

\bibliographystyle{abbrv}
\bibliography{biblio}

\end{document}